\documentclass[journal]{IEEEtran}
\usepackage{cite}

\ifCLASSINFOpdf
  \usepackage[pdftex]{graphicx}
\else
\fi
\usepackage{amsmath}
\usepackage{amsfonts}
\usepackage{enumitem}
\usepackage{algorithm, algorithmic}
\usepackage{url}

\usepackage{amsthm}

\newtheorem{theorem}{Theorem}
\newtheorem{lemma}{Lemma}

\newtheorem{conjecture}{Conjecture}

\theoremstyle{remark}
\newtheorem*{remark}{Remark}

\usepackage{multirow}

\usepackage[caption=false, font=footnotesize]{subfig}

\begin{document}
%

\title{MASS: Mobility-Aware Sensor Scheduling of Cooperative Perception for Connected \\Automated Driving}

%
%

\author{Yukuan~Jia,
        Ruiqing~Mao,
        Yuxuan~Sun,~\IEEEmembership{Member,~IEEE,}
        Sheng~Zhou,~\IEEEmembership{Member,~IEEE,}
        and~Zhisheng~Niu,~\IEEEmembership{Fellow,~IEEE}
\thanks{Yukuan Jia, Ruiqing Mao, Sheng Zhou, and Zhisheng~Niu are with Beijing National Research Center for Information Science and Technology, Department of Electronic Engineering, Tsinghua University, China. Emails: \{jyk20, mrq20\}@mails.tsinghua.edu.cn, \{sheng.zhou, niuzhs\}@tsinghua.edu.cn.}%
\thanks{Yuxuan~Sun is with School of Electronic and Information Engineering, Beijing Jiaotong University, China. Email: yxsun@bjtu.edu.cn}
\thanks{Part of this work has been published in IEEE ICC 2022 \cite{icc}.}
}
\maketitle

\begin{abstract}
Timely and reliable environment perception is fundamental to safe and efficient automated driving. 
However, the perception of standalone intelligence inevitably suffers from occlusions.
A new paradigm, Cooperative Perception (CP), comes to the rescue by sharing sensor data from another perspective, i.e., from a cooperative vehicle (CoV).
Due to the limited communication bandwidth, it is essential to schedule the most beneficial CoV, considering both the viewpoints and communication quality.
Existing methods rely on the exchange of meta-information, such as visibility maps, to predict the perception gains from nearby vehicles, which induces extra communication and processing overhead.
In this paper, we propose a new approach, learning while scheduling, for distributed scheduling of CP.
The solution enables CoVs to predict the perception gains using past observations, leveraging the temporal continuity of perception gains.
Specifically, we design a mobility-aware sensor scheduling (MASS) algorithm based on the restless multi-armed bandit (RMAB) theory to maximize the expected average perception gain.
An upper bound on the expected average learning regret is proved, which matches the lower bound of any online algorithm up to a logarithmic factor.
Extensive simulations are carried out on realistic traffic traces. 
The results show that the proposed MASS algorithm achieves the best average perception gain and improves recall by up to 4.2 percentage points compared to other learning-based algorithms.
Finally, a case study on a trace of LiDAR frames qualitatively demonstrates the superiority of adaptive exploration, the key element of the MASS algorithm.
\end{abstract}

\begin{IEEEkeywords}
Cooperative perception, mobility-aware, sensor scheduling, restless multi-armed bandit.
\end{IEEEkeywords}

%
\IEEEpeerreviewmaketitle

\section{Introduction}
Automated driving (AD) has received fast-growing attentions in recent years.
Among the enabling technologies of AD, reliable and timely perception is the basis for safety and energy efficiency.
Much effort has been made to improve the object detector, using state-of-the-art neural networks and multi-modality sensors \cite{bevformer,mv3d,survey-3ddetection,survey-multimodal}. 
However, the standalone perception has an intrinsic flaw because it can only provide line-of-sight (LoS) information with onboard sensors, and thus the traffic participants occluded in the blind zone cannot be detected. 
Thanks to the Vehicle-to-Everything (V2X) communication technology \cite{v2x}, vehicles can communicate with each other and a wide range of V2X applications are made possible, including \emph{cooperative perception} (CP).
By exchanging sensor data with other \emph{cooperative vehicles (CoVs)}, the occluded objects can be detected, and the perception quality is essentially improved.

There are three levels of CP, namely raw-level, feature-level, and object-level, differentiated by the type of the shared sensor data.
In its primitive form, object-level CP, only the list of detected objects is broadcast to other CoVs.
The messages are relatively lightweight, but the loss of details in sensor data causes difficulties in merging noisy, discrepant results from multiple sources \cite{fusioneye}.
In the raw-level CP \cite{cooper,autocast}, raw sensor data such as LiDAR point clouds and images are transmitted to other CoVs, which preserves complete context information.
However, the data volume to be transmitted in the raw-level CP is extremely high.
Feature-level CP \cite{f-cooper,v2vnet} strikes a balance between the above two in terms of the communication load, via extracting key features using neural networks.
As reported in Refs. \cite{cooper,tits-multi-schemes}, there are hard objects unrecognizable from any viewpoint alone but can be identified only when the raw data are aggregated. 
This implies the significance of context information and calls for an advanced CP system framework where more comprehensive sensor data, in raw-level or feature-level, are shared among CoVs.


Current research on CP has been focusing on designing novel fusion architectures \cite{v2vnet, where2com}, with limited attention paid to practical challenges such as the scarce V2X communication bandwidth and the high mobility of vehicles. 
For Cellular-V2X technology, the allocated bandwidth is 20MHz in China and 30MHz in US \cite{5gaa}. 
However, the real-time streaming of high-definition video or LiDAR point clouds takes several megabytes per second for one single link, which is not scalable subject to the total bandwidth constraint.
It is hardly feasible for the V2X network to support raw-level or feature-level sensor data broadcast by all CoVs.
Therefore, the scheduling of sensors in unicast CP, i.e., \emph{whom to cooperate with}, is a challenging and important problem yet to solve.

High mobility is another under-addressed issue in the literature of CP.
When multiple CoVs are available, it is challenging to accurately determine the perception gain of a viewpoint due to unclear occlusion relationships, heterogeneous sensor qualities, as well as the black-box nature of neural networks.
The movements of CoV also lead to a dynamic candidate CoV set over time-varying network topologies.
Many architectures rely on extra metadata messages to gather clues about the perception gain of additional sensors.
For example, the future trajectories \cite{autocast} or confidence maps \cite{where2com} are exchanged for spatial reasoning and perception gain prediction.
However, the sensor viewpoints and wireless channels are time-varying, making the prediction stale very quickly.

These challenges motivate our solution, \emph{learning while scheduling}, in the framework of \emph{online learning}.
Specifically, we harness the mobility and make scheduling decisions based on the historical perception gains from other CoVs' sensor data, with notably less scheduling overhead compared to existing methods.
This leads to a Multi-Armed Bandit (MAB) typed problem, which requires exploring to learn about the environment and simultaneously exploiting the empirically optimal action.
Its basic form, with stationary reward distribution, is well solved by the Upper Confidence Bound (UCB) algorithm with performance guarantee \cite{ucb}.
The UCB algorithm has been applied in a wide range of areas, including edge computing \cite{offload} and mobility management \cite{emm} in wireless networks.
Our problem, on the other hand, falls within the domain of Restless Multi-Armed Bandit (RMAB) \cite{rmab}, where the rewards are constantly evolving due to mobility.

To deal with the RMAB problem, discounted UCB and sliding-window UCB \cite{sw-ucb} are adapted from the classic UCB algorithm by introducing forgetting mechanisms.
An activation-based policy is proposed for Brownian restless bandit \cite{restless}, which leverages statistical assumptions of rewards.
Moreover, the Exp3 \cite{exp3} algorithm combats adversarial rewards that can change arbitrarily.
However, none of the existing methods is engaged with an \emph{ever-changing} set of CoV candidates, as is the case when CoVs have different destinations.
Our main contributions are summarized as follows:
\begin{enumerate}
    \item We propose a \emph{learning while scheduling} framework based on the RMAB theory for the distributed scheduling of \emph{decentralized CP}. 
    Our framework enjoys the advantage of negligible communication and processing overhead, compared to existing solutions that require frequent exchange of meta-information.
    \item A novel mobility-aware sensor scheduling (MASS) algorithm is proposed to leverage the dynamics in perception gains due to vehicular mobility.
    It is also proved that MASS effectively balances exploration and exploitation with a dynamic set of CoV candidates, where the learning regret matches the lower bound of any online algorithm up to a logarithmic factor.
    \item Extensive simulations along with supporting empirical studies are carried out, showing that the proposed MASS algorithm outperforms other online scheduling policies.
    A case study on a trace of LiDAR frames is also provided to qualitatively illustrate the benefit of adaptive exploration, which serves as the key element of MASS.
\end{enumerate}

The rest of this paper proceeds as follows.
We first brief on the related works in Section \ref{sect:related-work}.
Then the system model and problem formulation are introduced in Section \ref{sect:system-model}. 
In Section \ref{sect:algorithm}, the MASS algorithm is proposed, followed by the performance analysis conducted in Section \ref{sect:performance}. 
The experiment results are presented in Section \ref{sect:experiments}, and the paper is concluded in Section \ref{sect:conclusion}.






\section{Related Work} \label{sect:related-work}
\subsection{Dataflow of CP}
There is a wide range of CP systems designed with different architectures.
In this subsection, we divide them into three categories based on the dataflow of sensor data, namely broadcast, centralized, and decentralized.

\textbf{Broadcast}. 
It is most straightforward to share one's sensor data by broadcasting to its neighbor CoVs.
The broadcast of object-level data has been standardized as collective perception message (CPM) \cite{cpm} by ETSI, but suffers from poor scalability due to severe wireless interference and processing burden in the current C-V2X network \cite{cpm-genrules}.
Other new designs require the broadcast of bird eye's view (BEV) features, and the method of data fusion from multiple sources has been a research focus.
For example, DiscoNet \cite{disconet} is a teacher-student framework to learn the pose-aware, attention-based merging via knowledge distillation, where raw-level CP guides the feature-level CP.
A vision transformer-based architecture is proposed to capture the inter-agent relationships in V2X-ViT \cite{v2x-vit}.
V2VNet \cite{v2vnet} leverages graph neural networks (GNN) to achieve a multi-round aggregation of compressed feature data among nearby CoVs.
However, the volume of feature data still poses a major challenge for the V2X network, particularly subject to the limited data rate of broadcast transmission.

\textbf{Centralized}.
Road side units (RSUs) are deployed in the intelligent transportation system to provide a holistic, superior viewpoint from above, especially at urban intersections.
Therefore, the RSU can broadcast its detection results to nearby CoVs.
Following this idea, VIPS \cite{vips} specializes in object fusion based on efficient matching of graph structures, handling the time asynchrony and localization error.
On the other hand, the RSU can also serve as a fusion center to aggregate CPMs from nearby CoVs and broadcast the merged results \cite{citemy}.
Moreover, compressed raw point clouds and features can also be transmitted to the edge server to perform data fusion and detection, as in EMP \cite{emp} and VINet \cite{vinet}.
Although RSU is very effective in assisting perception, the deployment is costly and it is impossible to cover everywhere.

\textbf{Decentralized}.
Without a fusion center, the sensor data are exchanged in a decentralized, \emph{on-demand}, and \emph{unicast} manner among CoVs, which is more bandwidth-efficient than broadcast.
Decentralized CP can be realized in two different ways. 
In one way, a global scheduler determines the communication topology among CoVs based on the estimated rewards provided by the CoVs \cite{autocast}.
The transmission decision could also be determined independently by the CoVs themselves with locally available information \cite{where2com}.
In the following subsection, we review some recent works on the scheduling of decentralized CP.

\subsection{Scheduling of Decentralized CP}
Under network bandwidth constraints, the sensor scheduling problem of CP arises naturally.
AutoCast \cite{autocast} predicts the visibility and relevance of detected objects to other CoVs, using meta-information including their future trajectories.
Then a global scheduler allocates the communication bandwidth using a greedy max-weight scheduler.
Ref. \cite{action-branch} uses reinforcement learning techniques to address the CoV association problem by RSU, based on the CoVs' interest in locations.
Similar to the centralized architecture, these methods also rely on RSU to coordinate the transmissions.
Besides, Ref. \cite{platoon-cp} groups vehicles into a cooperative platoon, incorporates task offloading into CP, and solves the joint optimization problem.
Recently, the paradigm of multi-agent collaboration has been applied.
For example, When2com \cite{when2com} applies a three-stage handshake mechanism between any two CoVs to decide whether the cooperation is necessary.
Request maps are exchanged in Where2com \cite{where2com}, which then utilizes spatial confidence to determine the communication graph.

To summarize, the scheduling of decentralized CP depends on spatial reasoning or attention mechanism, both of which require extra communication and computation.
In fact, timeliness is one of the most crucial factor for environmental perception in automated vehicles. 
Excessive cooperation results in network latency, leading to severe performance drop due to positional drifts \cite{dair-v2x}.
In our previous work \cite{icc}, an algorithm is proposed to gradually learn the perception gain under the quasi-stationary assumption, which enjoys negligible scheduling overhead.
While preserving the real-time advantage, this paper deals with the dynamics of the perception gains due to vehicular mobility, which is more practical.

\section{System Model and Problem Formulation} \label{sect:system-model}
\subsection{System Overview}
We consider a connected automated driving scenario where a number of CoVs are connected via the V2V network.
For safe and efficient self-driving, perception with timeliness and high reliability is crucial to vehicles.
In this system, CP is supplementary to the standalone perception.
As the basis, standalone perception is carried out with only onboard sensors and performs object detection at a high frequency to ensure timeliness. 
On the other hand, CP merges the sensor data from multiple sources with the local data to deal with blind areas and long-tail cases.
A CoV can simultaneously request sensor data from others and transmit its sensor data to others on demand.

Typically, the CoVs periodically broadcasts short beacon messages such as Cooperative Awareness Messages (CAMs) \cite{cam} that include vehicle states and other optional information.
To enable CP, a bit indicating sensor sharing functionality and a list of indicators for available data formats are added to the beacon message, thereby other CoVs can request sensor data in a compatible data format.
The sensor data formats can be raw images, raw point clouds, or intermediate features extracted by neural networks.
To save the communication bandwidth, we assume that the CoVs can receive sensor data from only one of the peers at a time.
Based on the received beacons, each CoV sends a sensor request to one of the other CoVs for additional sensor data to augment its perception.
Furthermore, the amount of communicated sensor data is subject to the available data rate of wireless links.

\begin{figure}[!t]
	\centering
	\includegraphics[width=0.45\textwidth]{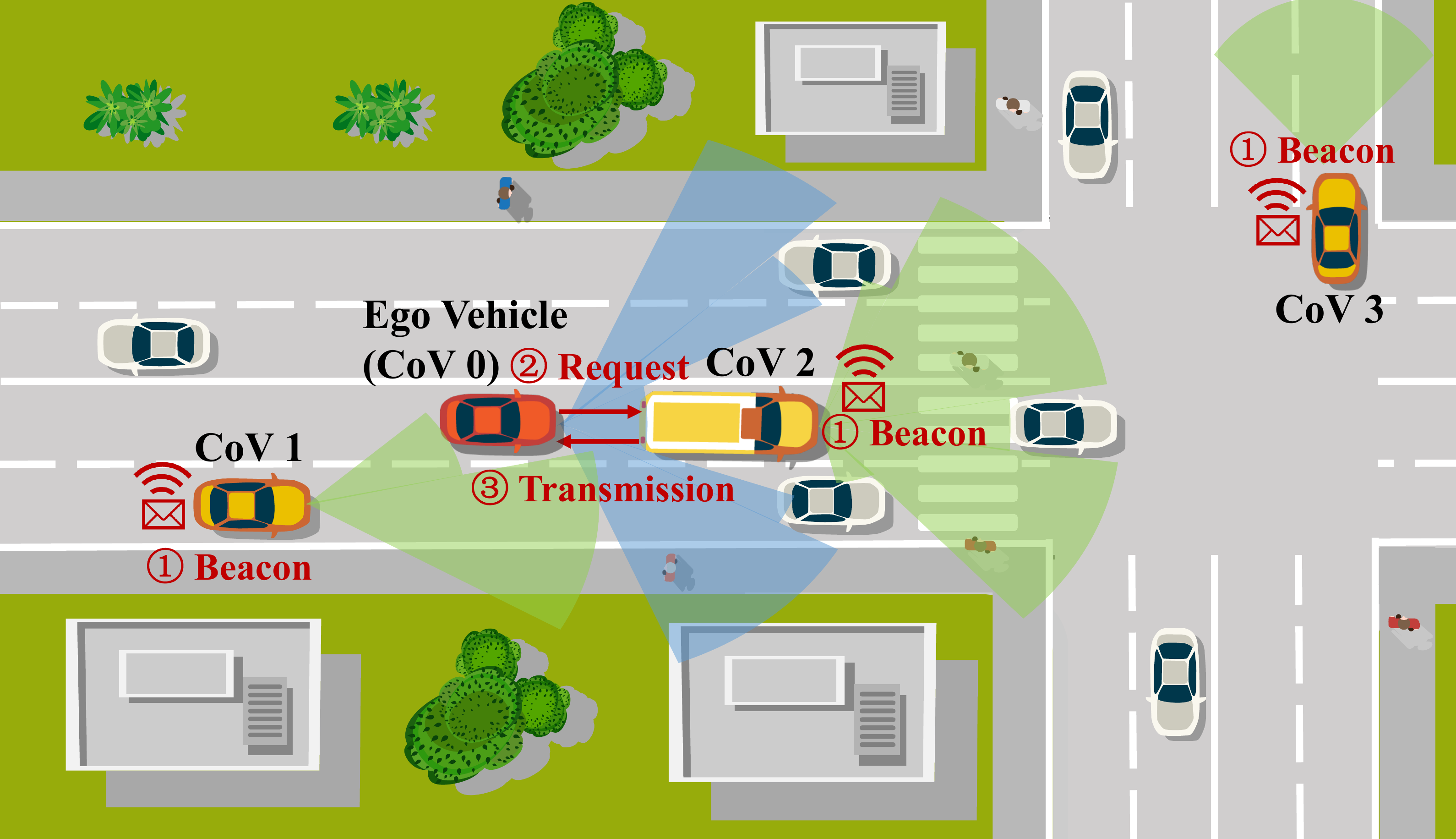}
	\caption{Illustration of the sensor scheduling procedure in CP.}
	\label{selection-idea}
\end{figure}


In this work, we investigate the distributed sensor scheduling problem in CP, where each CoV independently schedules another CoV to transmit its sensor data, without the latency induced by coordination.
Therefore, we focus on one particular CoV and refer to it as the \emph{ego vehicle} below.
As illustrated in Fig. \ref{selection-idea}, the ego vehicle (CoV 0) is driving on the street and aims to detect the surrounding objects using a CP framework.
It receives beacon messages from three other CoVs, i.e., CoV 1 to 3, that can offer sensor data from different views and identify them as candidates.
The ego vehicle then decides to schedule CoV 2 because it has a richer supplementary view and a LoS communication link.
Finally, the transmitted sensor data are merged with the onboard sensor data and fed into the detector.
The optimality of sensor sharer is determined by many factors, such as occlusion relationships, wireless link states, sensor qualities and configurations.
Therefore, it is challenging to optimize scheduling decisions.

\subsection{Procedure of CP}
We consider a discrete-time system, with the length of a time slot $\Delta t$ equal to the period of CP.

\textbf{Candidate Discovery:}
At the beginning of time slot $t$, the ego vehicle identifies the nearby CoVs that can offer compatible formats of sensor data based on received beacon messages.
The ego vehicle then determines the set of candidate sharers $\mathcal{V}_t$ by filtering the CoVs outside a specific range to restrict the total number of candidates under $V_\mathrm{max}$. 
The total bandwidth of V2X communications is denoted by $W$.
Depending on the network status and the use of other V2X applications, the ratio of available communication resources $\eta_i(t)$ is typically time-varying and different across CoVs.
According to Shannon's formula, the data transmission rate from CoV $i$ to the ego vehicle at time slot $t$ is expressed as
\begin{align}
    r_i(t) = W \eta_i(t) \log_2\left(1+\frac{P h_i(t)}{\sigma_n^2}\right),
\end{align}
where $P$ is the transmit power, $h_i(t)$ denotes the channel gain between CoV $i$ and the ego vehicle, and $\sigma_n^2$ is the noise power.

\textbf{Standalone Perception:}
At time slot $t$, the set of interested objects in the proximity of the ego vehicle is denoted by $\mathcal{O}_t$.
With onboard sensors, the ego vehicle captures a frame of data $\mathcal{X}_0^{(t)}$, which can be a raw image or a frame of LiDAR point clouds.
The part of the data that relate to object $j\in\mathcal{O}_t$ is denoted by $\mathcal{X}_{0,j}^{(t)} \subset \mathcal{X}_0^{(t)}$.
Then the sensor data are fed into an object detector to obtain the detection results, i.e., $\Phi \left(\mathcal{X}_{0,j}^{(t)}\right)\in \{0,1\}$.
If $\mathcal{X}_{0,j}^{(t)}$ contains enough information of object $j$ for accurate classification and localization, the detector can correctly detect the object $j$, i.e., $\Phi\left(\mathcal{X}_{0,j}^{(t)}\right)=1$, otherwise $\Phi\left(\mathcal{X}_{0,j}^{(t)}\right)=0$.
Particularly, if an interested object $j$ is completely occluded, then $\mathcal{X}_{0,j}^{(t)}=\emptyset$ and $\Phi\left(\mathcal{X}_{0,j}^{(t)}\right)=0$.

\textbf{Cooperative Perception:}
With CP, the ego vehicle schedules a candidate CoV $i\in\mathcal{V}_t$ and requests for its sensor data that could be raw sensor data or features extracted from neural networks.
This additional sensor data can contain finer textures of objects, provide a better perspective or even reveal invisible objects from the ego view.
Assume the CoVs are well-synchronized with perfect pose information, then it is very likely that the shared data could help identify missed detections.
To guarantee the timeliness of perception results, the scheduled CoV has to compress the sensor data subject to the transmission rate, otherwise the high latency would compromise the benefit of CP.
For example, the raw sensor data are uniformly down-sampled to a lower resolution, while fewer channels of the feature data are transmitted.
More efficient compression algorithms are beyond the scope of this paper.
Before transmission, the original sensor data $\mathcal{X}_i^{(t)}$ are compressed to $\Tilde{\mathcal{X}}_i^{(t)}$, satisfying
\begin{align}
    D\left(\Tilde{\mathcal{X}}_i^{(t)}\right) \le r_i(t) \Delta t,
    \label{formula:rate}
\end{align}
where $D\left(\cdot\right)$ denotes the data size in bits.
Finally, the detection result of object $j$ is $\Phi\left(\mathcal{X}_{0,j}^{(t)} \cup \Tilde{\mathcal{X}}_{i,j}^{(t)}\right) \in \{0,1\}$ for CP with CoV $i$, where $\Tilde{\mathcal{X}}_{i,j}^{(t)}$ denotes the part of compressed sensor data that relate to object $j$.

\textbf{Perception Gain Evaluation:} 
One of the usual metrics for object detection is \emph{recall}, defined as the number of detected objects divided by the total number of objects.
In the automated driving scenario, objects can have different importance to the ego vehicle, and thus we consider the sum of importance weight for missed objects in the metric.
Specifically, the perception cost of standalone intelligence is given by
\begin{align}
    c_0(t) = \sum_{j\in\mathcal{O}_t} w_j^{(t)} \left(1-\Phi \left(\mathcal{X}_{0,j}^{(t)}\right)\right),
\end{align}
where $w_j^{(t)}$ is the importance weight of object $j$ at time slot $t$.
The importance weights relate to the distance or the future trajectory of the ego vehicle, and it is up to the implementation.
As a special case, minimizing the perception cost is equivalent to maximizing the recall when the importance weights are equal.
With the shared data from the CoV $i$, the ego vehicle performs CP to augment the standalone perception, and the cost is
\begin{align}
    c_i(t) = \sum_{j\in\mathcal{O}_t} w_j^{(t)} \left(1-\Phi \left(\mathcal{X}_{0,j}^{(t)}\right)\right) \left(1-\Phi\left(\mathcal{X}_{0,j}^{(t)} \cup \Tilde{\mathcal{X}}_{i,j}^{(t)}\right)\right).
\end{align}
The cost of cooperative perception is determined by not only the viewpoint but also the transmission rate that decides the quality of shared data.
A major challenge in the evaluation of perception quality is that without the knowledge of ground truths $\mathcal{O}_t$, the actual perception costs $c_0(t)$ and $c_i(t)$ are unknown to the ego vehicle.
Nevertheless, the perception gain from CoV $i$ can be calculated by 
\begin{align} \label{formula:gain}
    g_i(t) &= c_0(t)-c_i(t) \nonumber \\
    &= \sum_{j\in\mathcal{O}_t} w_j^{(t)} \left(1 - \Phi\left(\mathcal{X}_{0,j}^{(t)}\right)\right) \Phi\left(\mathcal{X}_{0,j}^{(t)} \cup \Tilde{\mathcal{X}}_{i,j}^{(t)}\right).
\end{align}
In other words, the gain is the sum of the weighted costs for newly detected objects.
However, this perception gain from the scheduled CoV is available only when the detection is finished at the end of a time slot.

\subsection{Problem Formulation}
Consider the ego vehicle driving on a trip over $T$ time slots.
During the trip, other traffic participants travel alongside the ego vehicle and form highly dynamic occlusion relationships and volatile wireless links.
Among them, the CoVs may come close and leave at certain times during the trip.
Consequently, the set of candidate CoVs $\mathcal{V}_t$ and their perception gain $g_i(t)$ are constantly evolving.
The objective is to maximize the total perception gain of cooperative perception by optimizing the scheduling decisions of the ego vehicle. 
The sensor scheduling problem is formulated as
\begin{align} \label{formula:problem}
    \max_{a_1,\dots,a_T} &\quad \frac{1}{T} \sum_{t=1}^T g_{a_t}(t), \\
    \mathrm{ s.t. } &\quad a_t \in \mathcal{V}_t,
\end{align}
where $a_t$ is the optimization variable, representing the index of the scheduled CoV in time slot $t$.

Note that there are some subtleties in this problem.
Firstly, since there is no complete information of everything on the road, it is usually hard to directly estimate the gain $g_i(t)$ due to unpredictable occlusions.
Moreover, only the perception gain of scheduled CoV, i.e., $g_{a_t}(t)$, is available at the end of time slot $t$, while the other unscheduled CoVs are not observed.
Therefore, the offline optimal solution to (\ref{formula:problem}), i.e.,
\begin{align}
    a^*_t = \arg\max_{a_t\in\mathcal{V}_t} g_{a_t}(t),
\end{align}
is not feasible in practice.

Nevertheless, we seek to leverage the \emph{temporal continuity} of $g_i(t)$ and propose to learn the gains online from historical observations $g_{a_1}(1)$, $g_{a_2}(2),\dots,g_{a_{t-1}}(t-1)$.
Specifically, due to the velocity limit, the relative positions among CoVs and their perspectives cannot drift too much from time slot to time slot.
This type of problem falls in the category of the restless multi-armed bandits (RMAB) \cite{rmab}, where the CoVs correspond to the arms.
The change of $g_i(t)$ from the last observation requires the algorithm to schedule every candidate CoV once in a while to gain knowledge of its latest perception gain.
This process is termed \emph{exploration}.
On the other hand, the empirically optimal CoV should be scheduled frequently to \emph{exploit} the knowledge.
The target is to learn and schedule the optimal CoV $a^*_t$ while balancing exploration and exploitation.

\section{MASS: Mobility-Aware Sensor Scheduling Algorithm} \label{sect:algorithm}
In this section, we develop an online learning-based algorithm to tackle the decentralized sensor scheduling problem in CP. 
Compared with existing solutions, our algorithm considers the dynamics of perception gains due to vehicular mobility and has negligible communication and computation overhead.

It is well known that the Upper Confidence Bound (UCB) algorithm \cite{ucb} is an optimal solution to stationary multi-armed bandits problems, where the rewards of arms have stationary distributions.
The UCB algorithm gradually eliminates the randomness in the rewards and schedules the most promising arm in each time slot.
On the contrary, our problem deals with \emph{dynamics} rather than \emph{randomness}.
The intuition is that, rather than shrinking the confidence bound with more trials, we should enlarge the confidence bound for idle CoVs over time in this problem.

We propose a low-complexity yet effective algorithm named MASS, as described in Algorithm \ref{alg:mass}, to schedule CoVs to share their sensor data.
This algorithm maintains confidence bounds of perception gains for all candidate CoVs and schedules the CoV with the maximum upper confidence bound.
The confidence bounds are computed based on the historical observations and properties of the assumed underlying process $g_i(t)$. 
Since the variation of perception gains is mainly influenced by frequent blockages due to mobility, without loss of generality, we approximate the increment of $g_i(t)$ by independent, normally distributed random variables.
Therefore, the confidence bound of a nearby CoV is proportional to the square root of its idle time.
The scale of confidence bounds is specified by a parameter $\beta$ that depends on the rate of change in $g_i(t)$ and the confidence level.
When $\beta$ takes a larger value, the algorithm is more aggressive in exploration.

\begin{figure}[!t]
    \renewcommand{\algorithmicrequire}{\textbf{Parameter:}}
    \begin{algorithm}[H]
    \begin{algorithmic}[1]
    	\REQUIRE $\beta$ 
    	\FOR {$t=1,\cdots,T$}
    	    \STATE Determine the candidate set of available CoVs $\mathcal{V}_t$.
    	    \IF {any CoV $i\in\mathcal{V}_t$ has not been scheduled}
                \STATE Schedule CoV $i$, i.e., $a_t=i$.
    	    \ELSE
    	        \STATE Calculate the upper confidence bound for each CoV $i\in\mathcal{V}_t$: 
    	        \begin{align}
    	            \Tilde{g}_i(t) = \hat{g}_i + \beta \sqrt{t-\tau_i}.
    	        \end{align}
    	        \STATE Schedule the CoV $a_t$ with the maximum upper confidence bound, ties broken arbitrarily: 
    	        \begin{align}
    	             a_t = \arg\max_{i\in\mathcal{V}_t} \Tilde{g}_i(t).
    	        \end{align}
    	    \ENDIF
                \STATE Send a request to CoV $a_t$ for sensor data, compressed subject to (\ref{formula:rate}) if the transmission rate is inadequate.
                \STATE Receive the sensor data, run the object detector, and evaluate the perception gain $g_{a_t}(t)$ by (\ref{formula:gain}).
                \STATE Update the last-seen gain of CoV $a_t$:
                \begin{align}
                    \hat{g}_{a_t} = g_{a_t}(t).
                \end{align}
                \STATE Update the last-seen time of CoV $a_t$: 
                \begin{align}
                    \tau_{a_t}=t.
                \end{align}
            \ENDFOR
    \end{algorithmic}
    \caption{MASS: Mobility-Aware Sensor Scheduling Algorithm}
    \label{alg:mass}
    \end{algorithm}
\end{figure}

In Algorithm \ref{alg:mass}, Line 2 is the CoV discovery phase, when the ego vehicle determines the candidate set of available CoVs with maximum size $V_\mathrm{max}$, based on beacon messages.
In Lines 3-4, the newly available CoV is explored once by immediately sending a sensor data request if one exists.
Otherwise, in Line 6, we calculate the upper confidence bound $\Tilde{g}_i(t)$ of the perception gain for each CoV.
Specifically, $\hat{g}_i$ is the \emph{last-seen gain} of CoV $i$, and the padding function depends on the \emph{last-seen time} $\tau_i$ and the parameter $\beta$.
The motivation for using the last-seen gain rather than all observations is that dynamics is more significant than randomness due to the mobility of vehicles, which will be shown in Section \ref{sect:exp-sumo} through experiments.
The key intuition is that the reward of scheduling a CoV consists of its perception gain and the knowledge of its exact gain at $t$.
This knowledge is crucial in the problem since the perception gain is highly dynamic.
We must ensure a moderate level of exploration to avoid missing the optimal CoV.
Therefore, in Line 7, the ego vehicle optimistically schedules CoV $a_t$ with the maximum upper confidence bound.
In Lines 9-12, the sensor data, compressed if the transmission rate is inadequate, is then sent to the ego vehicle for sensor fusion, detection, and evaluation.
Finally, the last-seen gain and last-seen time of CoV $a_t$ are updated for future scheduling.

\textbf{Overhead Analysis:}
In the scheduling of decentralized cooperative perception, frequent metadata message exchanges and data processing introduces extra overhead.
This inevitably adds latency to the perception pipeline, which may compromise the actual perception quality.
We compare the communication and computation overhead of our proposed algorithm with existing centralized and distributed scheduling methods.
Assume there are $N$ CoVs that request and provide the sensor data simultaneously.
In the MASS algorithm, the few bits of data format information are piggybacked on the periodically broadcast beacons.
Consequently, there are totally $O(N^2)$ simple computations among $N$ CoVs and negligible communication overhead for scheduling decisions.

Conventional scheduling algorithms usually require much more detailed perceptual state information of all CoVs.
For example, an RSU-based algorithm in \cite{citemy} requires the visibility grid states from all CoVs for global scheduling.
This procedure takes at least $2N$ extra messages to communicate with the edge server and a computation load of $O(G N^2)$ for the deep reinforcement learning algorithm, where $G$ is the number of grids.
On the other hand, Who2com \cite{who2com} proposes a handshake mechanism for distributed scheduling.
Each CoV first broadcasts a request message with compressed sensor data, then computes the matching scores with the candidates' message.
It involves $O(N^2)$ neural network attention operations and $N^2$ extra messages to feedback the scores before initializing a connection.
Similarly, Where2com \cite{where2com} utilizes confidence-aware spatial maps to decide the most beneficial CoV.
Although it reduces the communication amount with the attention mechanism, multiple rounds of message exchange and fusion process still take much time.
By comparison, our proposed MASS algorithm enjoys the advantage of low overhead thanks to online learning, by harnessing the dynamics of the perception gain.

\section{Performance Analysis}
\label{sect:performance}
In this section, we characterize the perception performance of our proposed MASS algorithm in the fixed and dynamic CoV candidate scenarios, respectively.
\subsection{Assumptions}
For theoretical analysis, we first normalize the perception gain $g_i(t)$ to the fundamental interval $[0,1]$:
\begin{align}
    G_i(t) = \min\{g_i(t)/g_\mathrm{max}, 1\},
\end{align}
where $g_\mathrm{max}$ is a threshold to restrict the maximum possible perception gain.
We assume that $G_i(t)$ independently follows a Gaussian random walk with reflecting boundaries:
\begin{align} \label{formula:randomwalk}
    G_i(t+1) = f_I ( G_i(t) + X_i(t) ),
\end{align}
where $X_i(t)$ takes an i.i.d. sample from $\mathcal{N}(0,\sigma^2)$, and
\begin{align}
    f_I(x) = 
    \begin{cases} 
        x', \quad x'<1 \\
        2-x', \quad x'\ge 1,
    \end{cases}
\end{align}
where $x' \equiv x \pmod{2}$.
In practice, the perception gain is influenced by many factors, such as the instantaneous occlusion status, available data rate, and time-varying importance weights.
Therefore, it is appropriate to regard the increment of $G_i(t)$ as Gaussian random variables by the central limit theorem.
The standard deviation of the increment $\sigma$ reflects the rate of change in the traffic environment.
Generally, when $\sigma$ is smaller, better perception quality can be achieved since the perception gains are less dynamic.
In the analysis below, we will measure the performance with respect to $\sigma$.

According to the dynamics (\ref{formula:randomwalk}), $G_i(t)$ is ergodic and has uniform stationary distribution on $[0,1]$.
Without loss of generality, we also assume that the initial states $G_i(0)$ follow the stationary distribution.
In an RMAB problem, the performance of an algorithm is usually measured by a \emph{learning regret}, defined as the performance loss compared to the offline optimal solution.
Given a specific problem instance, the learning regret of a scheduling algorithm $\mathcal{A}$ by time slot $T$ is written as
\begin{align}
    R_\mathcal{A}(T) = \sum_{t=1}^T \left[G^*(t)-G_{a_t}(t)\right],
\end{align}
where $G^*(t) = G_{a_t^*}(t)$ is the normalized gain of the optimal CoV.
Since $G_i(t)$ is constantly shifting during the trip, no online algorithm can achieve an upper bound of learning regret sublinear to $T$.
Therefore, we define the \emph{expected average learning regret} of a scheduling algorithm $\mathcal{A}$ as
\begin{align}
    \Bar{R}_\mathcal{A} = \frac{1}{T} {\mathbb{E} \left[ R_\mathcal{A}(T) \right]},
\end{align}
where the expectation is taken over all sample paths of the stochastic processes $G_i(t)$.
In the following subsections, we derive the upper bounds on the expected learning regret of the MASS algorithm under both fixed and dynamic CoV candidate settings.

\subsection{Bounds on Two Fixed Candidate CoVs}
We first investigate a simplified problem by assuming there are two fixed candidate CoVs, i.e., $\mathcal{V}_t=\mathcal{V}=\{1, 2\}$.
As in \cite{restless}, define a function $f$ to be \emph{well-behaved} on an interval $[t_1, t_2]$ if 
\begin{align}
    |f(t)-f(t')| \le c_f\sqrt{|t-t'|}\sigma, \  \forall t,t'\in [t_1, t_2]\cap\mathbb{N},
\end{align}
where $c_f=\Theta(\log{\frac{1}{\sigma}})^{1/2}$ is a large enough constant to guarantee low violation probability.
Furthermore, we define a problem instance to be \emph{well-behaved near $t$} if the perception gains of all CoVs are well-behaved on the interval $[t-\sigma^{-2}, t+\sigma^{-2}]$.
Note that when $t \le \sigma^{-2}$ or $t \ge T-\sigma^{-2}$, the definition is on the interval $[t-\sigma^{-2}, t+\sigma^{-2}] \cap [0,T]$.
Define $E_t$ as the event that the problem is well-behaved near $t$.
The following lemma bounds the violation probability of well-behavedness.
\begin{lemma} \label{lem:wellbehave}
    Let $c_f = 3(\log{\frac{1}{\sigma}})^{1/2}$. 
    For a problem instance, the violation probability of the well-behavedness near $t$ satisfies
    \begin{align}
        P(\Bar{E}_t) < O(\sigma^{2.4}).
    \end{align}
\end{lemma}
\begin{proof}
See Appendix \ref{app:wellbehave}.
\end{proof}

At time $t$, define the \emph{leader} as the CoV with maximum last-seen gain.
Let $H^*(t)$ denote the gain of the leader, and $\tau_i(t)$ denote the last-seen time of CoV $i$ at time $t$.
We divide the average learning regret into two parts:
\begin{align}
    R_\mathcal{A}(T) &= R^*(T) + \sum_{i=1}^2 R_i(T),
\end{align}
where the first part denotes the gain difference between the optimal CoV and the leader, i.e.,
\begin{align}
    R^*(T) = \sum_{t=1}^{T} \left[G^*(t)-H^*(t)\right],
\end{align}
and the second part is the gain difference between the leader and the scheduled CoV, i.e.,
\begin{align}
    R_i(T) &= \sum_{t=1}^{T} \left[H^*(t)-G_{a_t}(t)\right] \\
    &= \sum_{t=1}^{T} \mathbb{I}(a_t=i) \left[H^*(t)-G_i(t)\right] \\
    &\le \sum_{t=1}^{T} \frac{H^*(\tau_i(t))-G_i(\tau_i(t))}{t-\tau_i(t)},
\end{align}
where we spread the regret incurred at time $\tau_i(t)$ over the subsequent idle time.

We set the algorithm parameter $\beta=5c_f$, and provide some basic deterministic properties of the proposed MASS algorithm conditioned on $E_t$.
\begin{lemma} \label{lem:properties}
    For a problem instance well-behaved near $t$, the MASS algorithm has the following properties:
    \begin{enumerate}
        \item[a)] The optimal CoV at time $t$ is scheduled no later than $t+1$.
        \item[b)] The leader at time $t$ is scheduled no later than $t+1$.
        \item[c)] The change of the leader's gain in one slot is lower bounded by
        \begin{align}
            H^*(t+1) - H^*(t) \ge - 2c_f\sigma.
        \end{align}
        \item[d)] The gain difference between the optimal CoV and the leader is bounded by 
        \begin{align} \label{formula:gain-diff}
            G^*(t)-H^*(t) < 5c_f\sigma. 
        \end{align}
    \end{enumerate}
\end{lemma}
\begin{proof}
See Appendix \ref{app:properties}.
\end{proof}

Let $\delta_i(t)=H^*(t)-G_i(t)$.
To bound $R_i(T)$, it is important to characterize a relationship between $\delta_i(\tau_i(t))$ and $t-\tau_i(t)$.
\begin{lemma}
    \label{lem:unscheduletime}
    For a problem instance well-behaved near t, with the MASS algorithm we have: \\
    a) if CoV $i$ is not optimal at $t$, then
    \begin{align}
        t-\tau_i(t) \ge \Omega(\delta_i(t)/\beta\sigma)^2, \label{formula:idle-subopt} \\
        \delta_i(\tau_i(t)) \le 2\delta_i(t)+O(c_f\sigma). \label{formula:delta-subopt}
    \end{align}
    b) If CoV $i$ is optimal at $t$, then
    \begin{align} \label{formula:delta-opt}
        \delta_i(\tau_i(t)) \le O(c_f\sigma).
    \end{align}
\end{lemma}
\begin{proof}
See Appendix \ref{app:unscheduletime}.
\end{proof}

The intuition behind Lemma \ref{lem:unscheduletime} is that after scheduling a sub-optimal CoV, the subsequent idle time is proportional to the quadratic of the sub-optimality gap.
It is important to note that Lemma \ref{lem:properties} and Lemma \ref{lem:unscheduletime} are deterministic, conditioned on the well-behavedness of the problem instance.
Next, we deal with the conditional probability and bound the expected learning regret of the MASS algorithm with two fixed candidate CoVs.
\begin{theorem}[Fixed Candidates]
    \label{thm:fixed}
    Let $\beta=15\sigma\log\sigma^{-1}$.
    For a sufficiently long trip $T \ge \Omega(\sigma^{-2})$, the expected average learning regret of MASS with two fixed candidate CoV is bounded by
    \begin{align}
        \Bar{R}_\mathcal{\text{MASS}} \le
        O\left(\sigma^2\log^3(1/\sigma)\right).
    \end{align}
\end{theorem}
\begin{proof}
See Appendix \ref{app:fixed}.
\end{proof}
\begin{remark} 
This theorem utilizes the uniform stationary distribution of ergodic process $G_i(t)$ to bound the expected average learning regret.
The key of the proof is that the gap between two CoVs, $G_1(t)-G_2(t)$, has a very high probability of being much larger than $\sigma$.
Then the expected regret can be bounded using Lemma \ref{lem:unscheduletime} by conditional probabilities.
Note that the expected average learning regret is lower bounded by $O(\sigma^2)$ for any online algorithm \cite{restless}.
Our algorithm is near-optimal in the sense that our upper bound matches the lower bound up to a logarithmic factor.
\end{remark}

\subsection{Bounds on Dynamic Candidate CoVs}
During the trip, the candidate set of available CoVs $\mathcal{V}_t$ changes occasionally.
In the following algorithm, we will divide the trip into periods by the arrival times of candidate CoVs, and bound the regret in each period using Theorem \ref{thm:fixed}.
\begin{theorem} [Dynamic Candidates] \label{thm:dynamic}
Let $\beta=15\sigma\log\sigma^{-1}$. 
For a sufficient long trip $T \ge \Omega(\sigma^{-2})$ with a dynamic candidate CoV set satisfying $|\mathcal{V}_t| \le 2$, the expected average learning regret of MASS is bounded by
    \begin{align}
        \Bar{R}_\mathcal{\text{MASS}} \le
        O\left(\sigma^2\log^3(1/\sigma)\right) + 2\lambda,
    \end{align}
    where $\lambda$ is the arrival rate of candidate CoVs.
    When $\lambda \le O\left(\sigma^2\log^3(1/\sigma)\right)$, the expected average learning regret is bounded by 
    \begin{align}
        \Bar{R}_\mathcal{\text{MASS}} \le O\left(\sigma^2\log^3(1/\sigma)\right).
    \end{align}
\end{theorem}
\begin{proof}
See Appendix \ref{app:dynamic}.
\end{proof}

It is much more difficult to prove the bound for any number of candidate CoVs since it is hard to guarantee the basic properties in Lemma \ref{lem:properties} without extra constraints.
A similar algorithm with an activation mechanism is proposed in \cite{restless}.
Specifically, when the upper confidence bound of a candidate is larger than the last-seen gain of the leader, it is activated until scheduled.
The earliest activated candidate is scheduled in odd-numbered time slots, while in even-numbered time slots the leader is exploited.
Although the additional rules facilitate a bound for any number of candidate CoVs, this algorithm is less efficient than our proposed MASS algorithm.
We will show through experiments in Section \ref{sect:exp-sumo} that this algorithm compromises the regret performance.
Moreover, we also conjecture that the regret bound also exists for our proposed MASS algorithm with any number of candidate CoVs, described as follows.
\begin{conjecture}
Let $\beta=15\sigma\log\sigma^{-1}$. 
For a sufficient long trip $T \ge \Omega(\sigma^{-2})$ with a dynamic candidate CoV set satisfying $|\mathcal{V}_t| \le V_\mathrm{max}$, the expected average learning regret of MASS is bounded by
    \begin{align}
        \Bar{R}_\mathcal{\text{MASS}} \le O\left(V_\mathrm{max}\sigma^2\log^3(1/\sigma)\right) + 2\lambda.
    \end{align}
    where $\lambda$ is the arrival rate of candidate CoVs.
    When $\lambda \le O\left(V_\mathrm{max}\sigma^2\log^3(1/\sigma)\right)$, the expected average learning regret is bounded by 
    \begin{align}
        \Bar{R}_\mathcal{\text{MASS}} \le O\left(V_\mathrm{max}\sigma^2\log^3(1/\sigma)\right).
    \end{align}
\end{conjecture}

\section{Experiments} \label{sect:experiments}
In this section, we first conduct empirical studies on the LiDAR-based object detection to characterize a relationship between the input sensor data $\mathcal{X}$ and the output of the object detector $\Phi(\mathcal{X})$.
Then based on the empirical model, extensive simulations are conducted to evaluate the perception gain of the proposed MASS algorithm, compared with baseline algorithms.
Finally, a case study is provided to compare the behavior of different online algorithms qualitatively.

\subsection{Empirical Studies on LiDAR-based Perception}
We conduct 3D object detection experiments on an open-source large-scale automated driving dataset, DOLPHINS \cite{dolphins}.
It is generated using the CARLA \cite{carla} traffic simulator, with a realistic environment rendered in six different scenarios, including intersections, highways, and T-junctions.
The dataset features the support for V2X, providing temporally-aligned sensor data from the ego vehicle, a collaborative vehicle, and an RSU.
There are 42,376 frames of sensor data with 3D bounding box labels for cars and pedestrians.
The ego vehicle is required to detect all the other traffic participants, including the occluded, within [-100m, 100m] in the driving direction and [-40m, 40m] in the perpendicular direction.
For generality and robustness, we focus on the raw-level sensor fusion, using point clouds scanned by LiDARs installed on the top of vehicles.
Based on the pose information, the raw point clouds are merged after the coordinate transformation.
The dataset is randomly split into the training, validation, and test sets with a 60\%:20\%:20\% ratio. 
We adopt a popular LiDAR-based 3D detection model, PointPillars \cite{pointpillars}, with pillar size 0.16m$\times$0.16m.
The intersection-over-union (IoU) threshold for accurate detection is set as 0.7 and 0.3 for cars and pedestrians, respectively.
With an open-source platform OpenPCDet \cite{openpcdet}, we train the model for 200 epochs using the one-cycle Adam optimizer.



\begin{figure}[!t]
    \centering
    \subfloat[]{\includegraphics[width=0.4\textwidth]{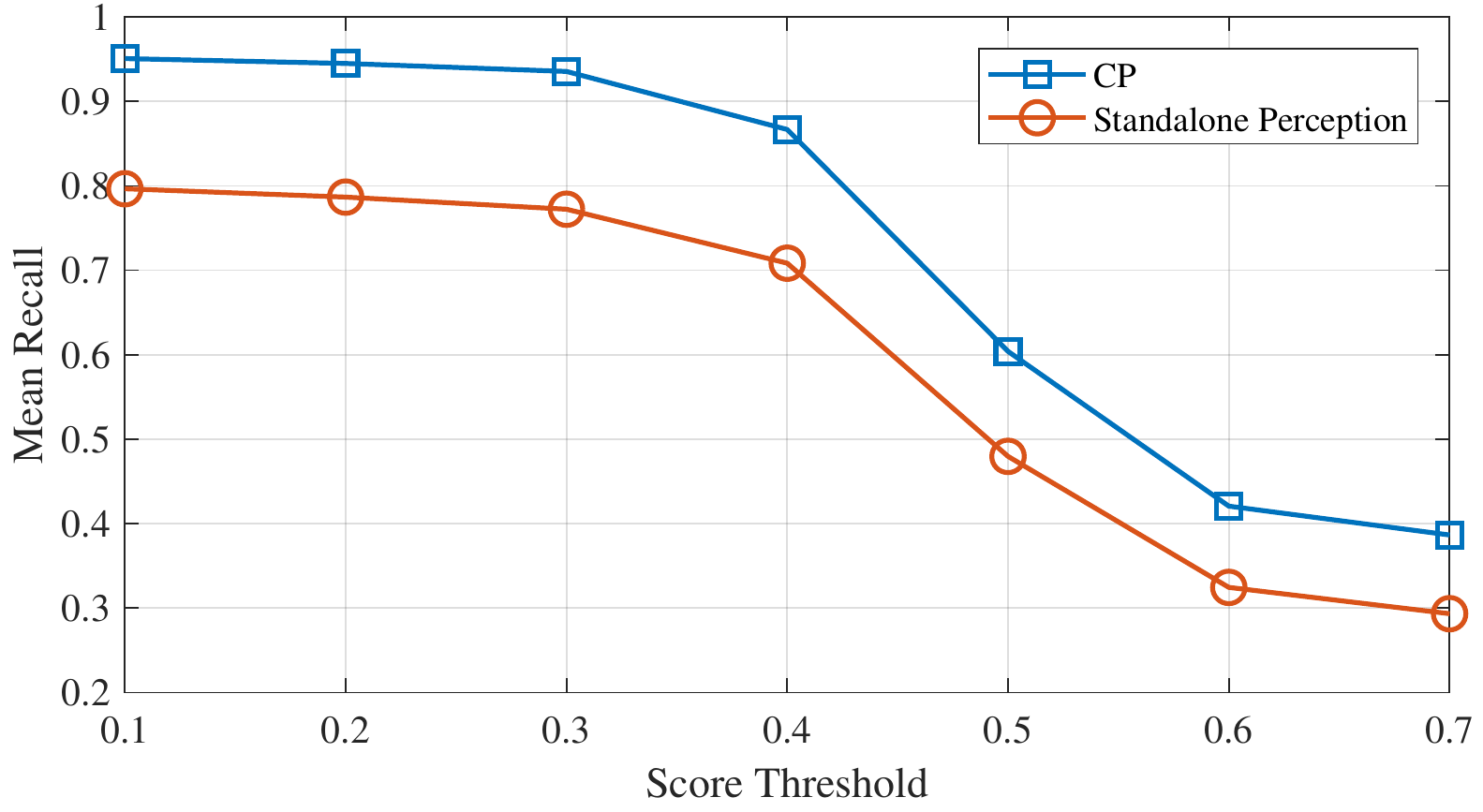}
    \label{fig:recall-score}}
    \hfill
    \subfloat[]{\includegraphics[width=0.4\textwidth]{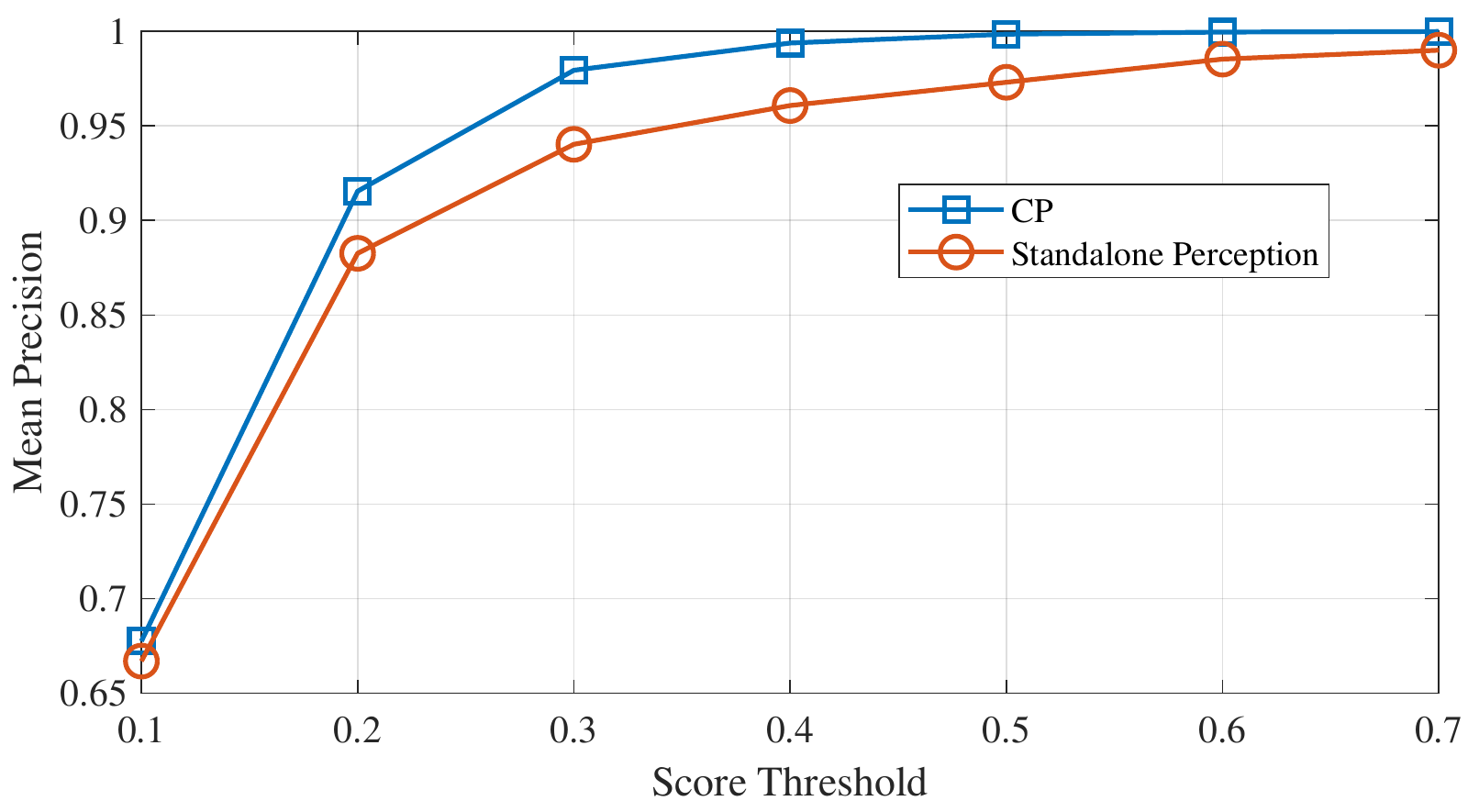}
    \label{fig:precision-score}}
    \caption{Performances of CP over standalone perception at different score thresholds. (a) The mean recall. (b) The mean precision.}
    \label{fig:score}
\end{figure}

We run the detection task on the test set and obtain mean recall and precision results over two categories, i.e., cars and pedestrians.
The score threshold for a positive detection sweeps from 0.1 to 0.7.
As shown in Fig. \ref{fig:recall-score}, the recall is significantly improved by the supplementary view, 
which shows the great potential of CP.
Besides, the precision, defined as the ratio of true detections to all detections, is slightly increased as well in Fig. \ref{fig:precision-score}.
Our evaluation of the perception gain depends on high precision since the gain is calculated based on the additional detections from CP.
Moreover, it is reasonable to observe the trade-off that when the score threshold increases, the precision enhances while the recall degrades.
To strike a balance between recall and precision, in the following experiments, the score threshold is set to 0.4.
With mean precision as high as 0.99 for CP, we can safely approximate the perception gain from the newly detected objects, neglecting the false positives.

\begin{figure}[!t]
	\centering
	\includegraphics[width=0.4\textwidth]{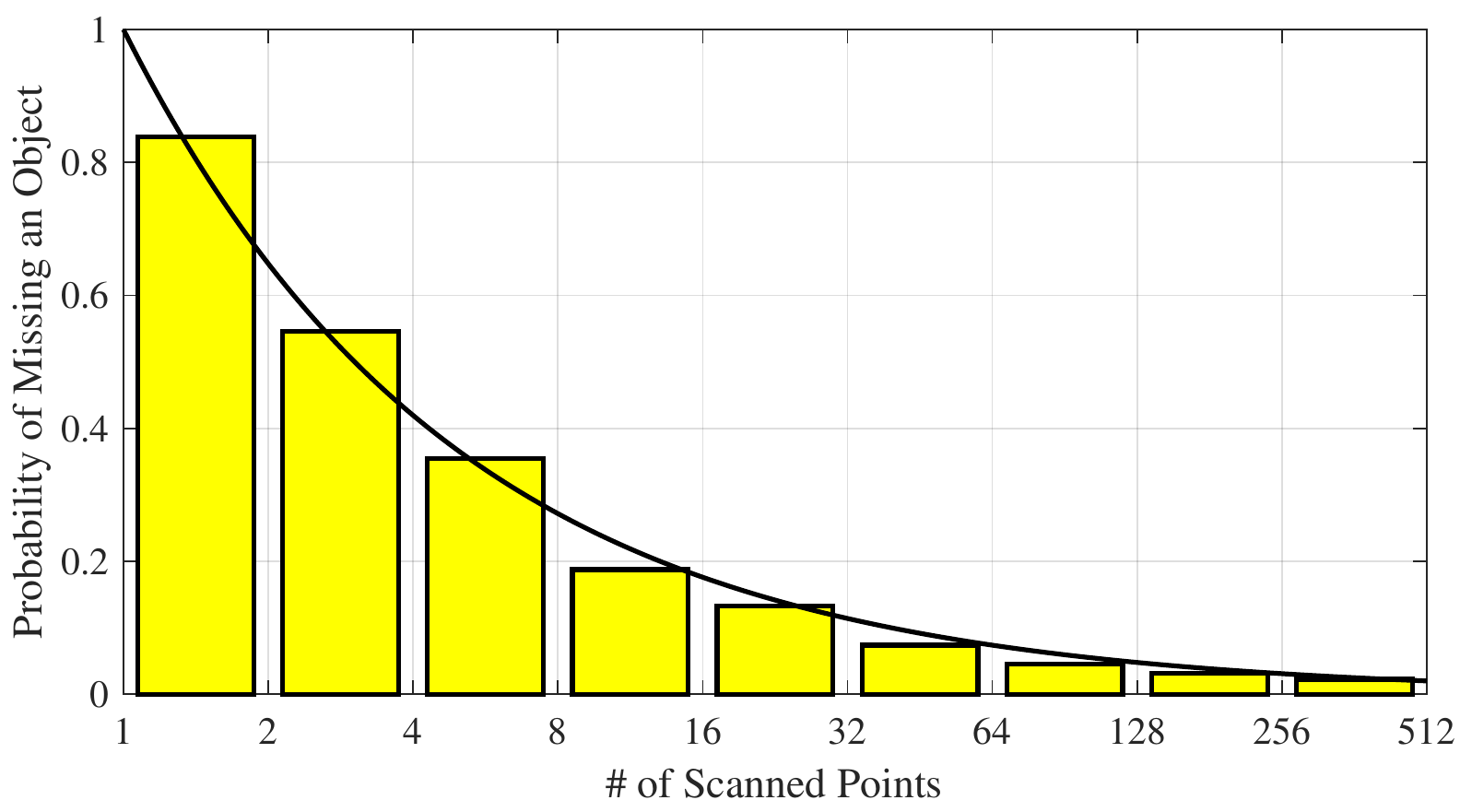}
	\caption{The empirical probability of missing an object with respect to the number of scanned points (in log scale).}
	\label{fig:detprob}
\end{figure}

Now turn to the relationship between the detection result and the scanned LiDAR points of an object.
As shown in Fig. \ref{fig:detprob}, the ground-truth objects are binned based on the number of points within the labeled 3D bounding boxes in log scale. 
Then we fit the missed detection probability to the exponential distribution and obtain the best fit with goodness $R^2=0.994$.
The statistics implies that the empirical probability of missing an object is approximately a power function of the number of scanned points on the object, i.e.,
\begin{align}
    \label{formula:fit}
    P\left(\Phi\left(\mathcal{X}_{0,j}^{(t)}\right)=1\right) 
    = e^{-0.4343 \log_2{N_\mathcal{X}}} 
    = N_\mathcal{X}^{-0.6265},
 \end{align}
where $N_\mathcal{X}$ denotes the number of points on an object.

Based on this observation, in the simulation of the following subsection, we will assume the detection result of an object is determined by the number of scanned points, neglecting other factors. 
Therefore, for each object $j\in\mathcal{O}_t$, a minimum number of scanned points $N_j$ are required for accurate detection, representing the \emph{difficulty} of the object.
This assumption is reasonable since a certain number of points are needed to exhibit the texture and shape information of a particular object.
By (\ref{formula:fit}), the object difficulty $N_j$ follows a long-tail zeta random distribution with cumulative distribution function
\begin{align}
    F_N(n) = n^{-0.6265}.
\end{align}
Then with CP, the total number of points on an object is increased, thus improving the chance of accurate detection.


\subsection{Simulation with Detection Model}
\label{sect:exp-sumo}

\begin{figure}[!t]
	\centering
	\includegraphics[width=0.4\textwidth]{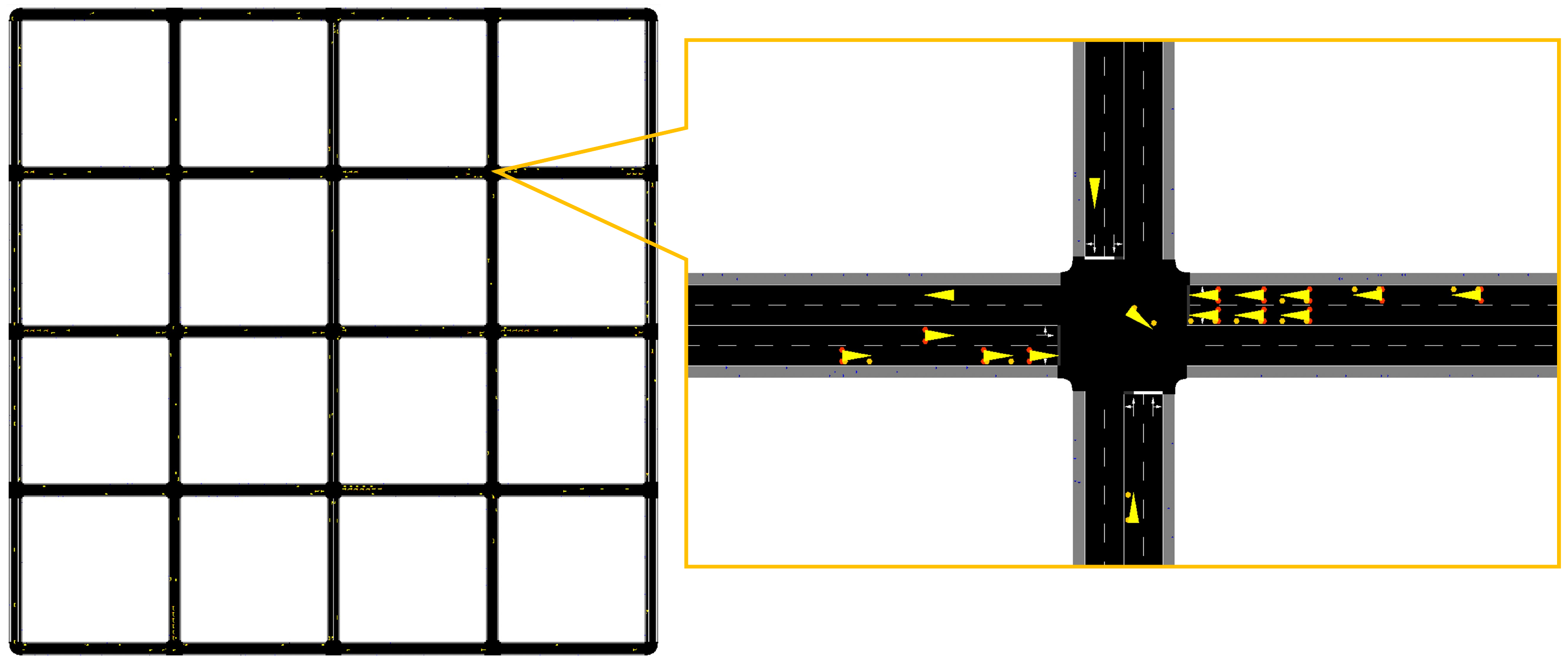}
	\caption{An illustration of the map created in SUMO with a zoomed snapshot at an intersection.}
	\label{fig:sumo}
\end{figure}

We first generate mobility traces in BEV using the microscopic traffic simulator SUMO \cite{sumo}, highlighting the sensor coverage and occlusion relationships.
A Manhattan-like map is created with a 4-by-4 grid, as shown in Fig. \ref{fig:sumo}.
The street is bidirectional, with two lanes and a sidewalk in each direction, and the side length of a block is $200$ meters.
The blocks represent the space for buildings, which occludes the sight of perpendicular directions, causing blind zones.
The traffic flow is controlled by the traffic lights at intersections, leading to queues of cars waiting to enter the intersection as well as some occlusions.
While the total number of cars is fixed, pedestrians are spawned randomly at the endpoints of each sidewalk and move towards the other endpoint as destinations. 
Based on the empirical modeling of perception in the subsection above, we randomly assign the difficulties, the minimum number of LiDAR points for correct detection, to each car and pedestrian.

We randomly select one vehicle as the ego vehicle, which aims to perceive the other traffic participants within 100m using CP.
The importance weight of objects is calculated by 
\begin{align}
    w_j^{(t)} = \begin{cases}
        1, &\quad d_j^{(t)} \le 10, \\
        2 - \log_{10}{d_j^{(t)}}, &\quad 10 < d_j^{(t)} < 100, \\
        0, &\quad d_j^{(t)} \ge 100,
    \end{cases}
\end{align}
where $d_j^{(t)}$ is the distance between the object $j$ and the ego vehicle.
Define \emph{CoV ratio} as the proportion of CoVs with sensor sharing functionality among vehicles on the road.
Each CoV, including the ego vehicle, is equipped with an omni-directional LiDAR on the top.
We simulate the laser scanning process in the experiment, considering the blockage effect of vehicles and buildings.
During the trip, the CoVs within 100 meters from the ego vehicle are identified as candidates since they are more likely to help reveal important objects.

Due to the decentralized congestion control (DCC) mechanism of the V2X network \cite{dcc}, the LiDAR point clouds are down-sampled when the communication bandwidth is inadequate.
For each CoV, the evolution of available communication resource ratio $\eta_i(t)$ is modeled by independent Markov chains with three states \cite{dcc-states}.
Besides, we adopt the V2V sidelink channel models in 3GPP TR 37.885 \cite{3gpp37885}, which introduces the NLOSv state, in which the direct path is blocked by vehicles.
Unlike the conventional NLOS channel that assumes blockage by larger objects such as buildings, the behavior of NLOSv channel is closer to a LOS channel with extra attenuation.
In the urban setting, the pathloss of the LOS and the NLOSv channels are specified by 
\begin{align}
    PL_\mathrm{LOS} = 38.77+16.7\log_{10}{d}+18.2\log_{10}{f_c},
\end{align}
where the NLoSv channel adds an extra blockage loss for each vehicle.
On the other hand, the pathloss of the NLOS channel is given by
\begin{align}
    PL_\mathrm{NLOS} = 36.85+30\log_{10}{d}+18.9\log_{10}{f_c}.
\end{align}
The simulation parameters are summarized in Table \ref{tab:param}.
We simulate for $T=10^4$ time slots, corresponding to a trip of 1,000 seconds.
An exemplary sample path of the perception gains is plotted in Fig. \ref{fig:samplepath}.

\begin{table}[!t]
\centering
\caption{Simulation Parameters}
\begin{tabular}{|ll|}
\hline 
\multicolumn{1}{|l|}{\textbf{Parameters}}                   & \textbf{Values}           \\ \hline\hline
\multicolumn{1}{|l|}{Length of Time Slot}              & 0.1s                            \\ \hline
\multicolumn{1}{|l|}{Number of Cars}              & 200                               \\ \hline
\multicolumn{1}{|l|}{CoV Ratio}              & 30\%                                    \\ \hline
\multicolumn{1}{|l|}{Speed Limit of Cars}              & 50km/h                               \\ \hline
\multicolumn{1}{|l|}{Turning Probabilities}             & 0.25 (Left), 0.25 (Right) \\ \hline
\multicolumn{1}{|l|}{Arrival Rate of Pedestrians}              & 0.2 persons/s                 \\ \hline
\multicolumn{1}{|l|}{Speed of Pedestrians}              & 1.2m/s                 \\ \hline
\multicolumn{2}{|c|}{\textbf{LiDAR-Related}}                                               \\ \hline
\multicolumn{1}{|l|}{\# of Lasers}                          & 16, 32, 64                        \\ \hline
\multicolumn{1}{|l|}{Vertical Field-of-view}             & 26.8$^{\circ}$                              \\ \hline
\multicolumn{1}{|l|}{Maximum Range}                     & 100m                               \\ \hline
\multicolumn{1}{|l|}{Height of Objects}                 & 1.7m                               \\ \hline
\multicolumn{1}{|l|}{Angular Resolution}    & 0.09$^{\circ}$                              \\ \hline
\multicolumn{1}{|l|}{Data Rate (64 channels)}         & 33.27Mbps                              \\ \hline
\multicolumn{2}{|c|}{\textbf{V2X-Related}}                                                      \\ \hline
\multicolumn{1}{|l|}{Carrier Frequency}                     & 5.9GHz                           \\ \hline
\multicolumn{1}{|l|}{Transmission Power}              & 23dBm                                \\ \hline
\multicolumn{1}{|l|}{Noise Power Spectral Density}          & -174dBm/Hz                              \\ \hline
\multicolumn{1}{|l|}{Receiver Noise Figure}            & 9dB                                 \\ \hline
\multicolumn{1}{|l|}{Shadowing Fading Std. Dev.}       & 3dB (LOS, NLOSv), 4dB (NLOS)         \\ \hline
\multicolumn{1}{|l|}{Vehicle Blockage Loss}            & max\{0, $\mathcal{N}(5,4)$\} dB     \\ \hline
\multicolumn{1}{|l|}{Channel Bandwidth}               & 30MHz                                \\ \hline
\multicolumn{1}{|l|}{Available Comm. Resource Ratio}            & 1.2MHz, 6MHz, 30MHz                  \\ \hline
\multicolumn{1}{|l|}{Transition Time of Resource Ratio}      & 10s                                \\ \hline
\end{tabular}
\label{tab:param}
\end{table}

\begin{figure}[!t]
	\centering
	\includegraphics[width=0.4\textwidth]{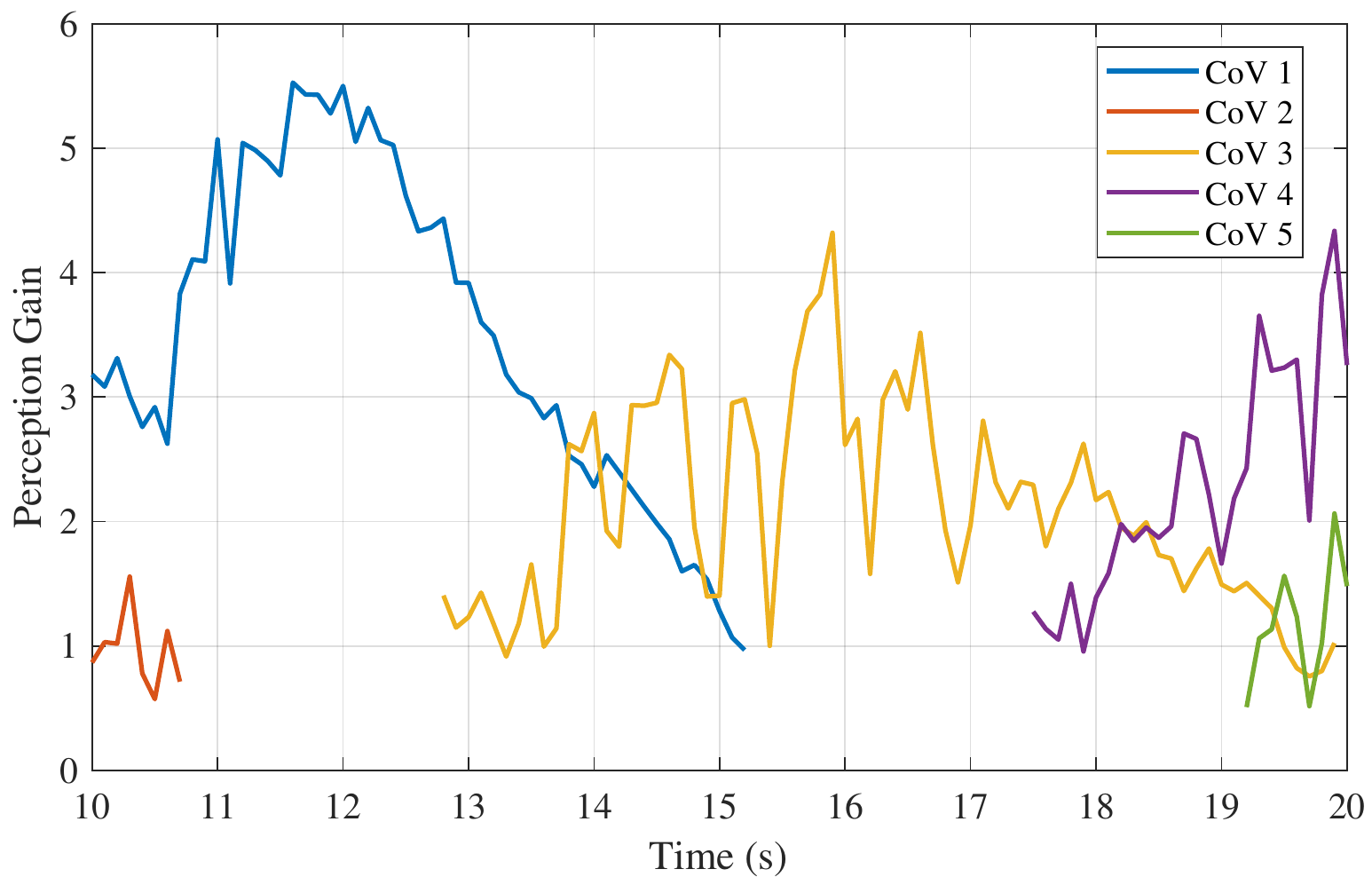}
	\caption{A sample path on the evolution of perception gains of nearby CoVs. The start point of a curve represents that a CoV is approaching the ego vehicle and becomes a candidate. The end point represents that the CoV is no longer a candidate.}
	\label{fig:samplepath}
\end{figure}

In the following, we compare the proposed MASS algorithm to four baselines:
1) \textbf{Closest CoV} is a naive policy without regard to historical observations.  
It is supported by the fact that the closest CoV usually has a good viewpoint for closer objects which have greater importance weights, and the pathloss is minimum.
2) In \textbf{Periodic ETC} (Periodic Explore-Then-Commit), the time is divided into epochs of equal length.
At the beginning of an epoch, each candidate is explored once, then for the rest of the epoch, the CoV with the maximum empirical perception gain is scheduled.
3) \textbf{SW-UCB} (Sliding Window UCB) \cite{sw-ucb} is adapted from the classic UCB algorithm, using the averaged observed rewards on a fixed-size horizon rather than the infinite horizon. 
4) \textbf{Earliest Activated} \cite{restless} is another online algorithm for restless bandit problem.
It explores the earliest activated CoV in odd-numbered time slots and exploits the leader in even-numbered time slots.
For a fair comparison, the algorithms are evaluated with sweeping parameters, summarized in Table \ref{tab:parameter}.

\begin{table}[!t]
\centering
\caption{Sweeping Parameters in Algorithms}
\label{tab:parameter}
\begin{tabular}{|l|l|l|}
\hline
\textbf{Algorithm}      & \textbf{Parameter} & \textbf{Range of Value}    \\ \hline \hline
Closest CoV             & -              & -                          \\ \hline
Periodic ETC            & Epoch Length   & \{2,3,...101\}             \\ \hline                           
\multirow{2}{*}{SW-UCB} & Horizon Length & \{5,10,20,30,40\}         \\ \cline{2-3}
                        & Scale of UCB $\beta$           & {$[10^{-1}, 10^1]$} (log scale) \\ \hline
Earliest Activation     & Scale of UCB $\beta$           & {$[10^{-1}, 10^{0.5}]$} (log scale) \\ \hline
MASS   & Scale of UCB $\beta$           & {$[10^{-0.9}, 10^{0.6}]$} (log scale) \\ \hline
\end{tabular}
\end{table}

\begin{figure}[!t]
	\centering
	\includegraphics[width=0.4\textwidth]{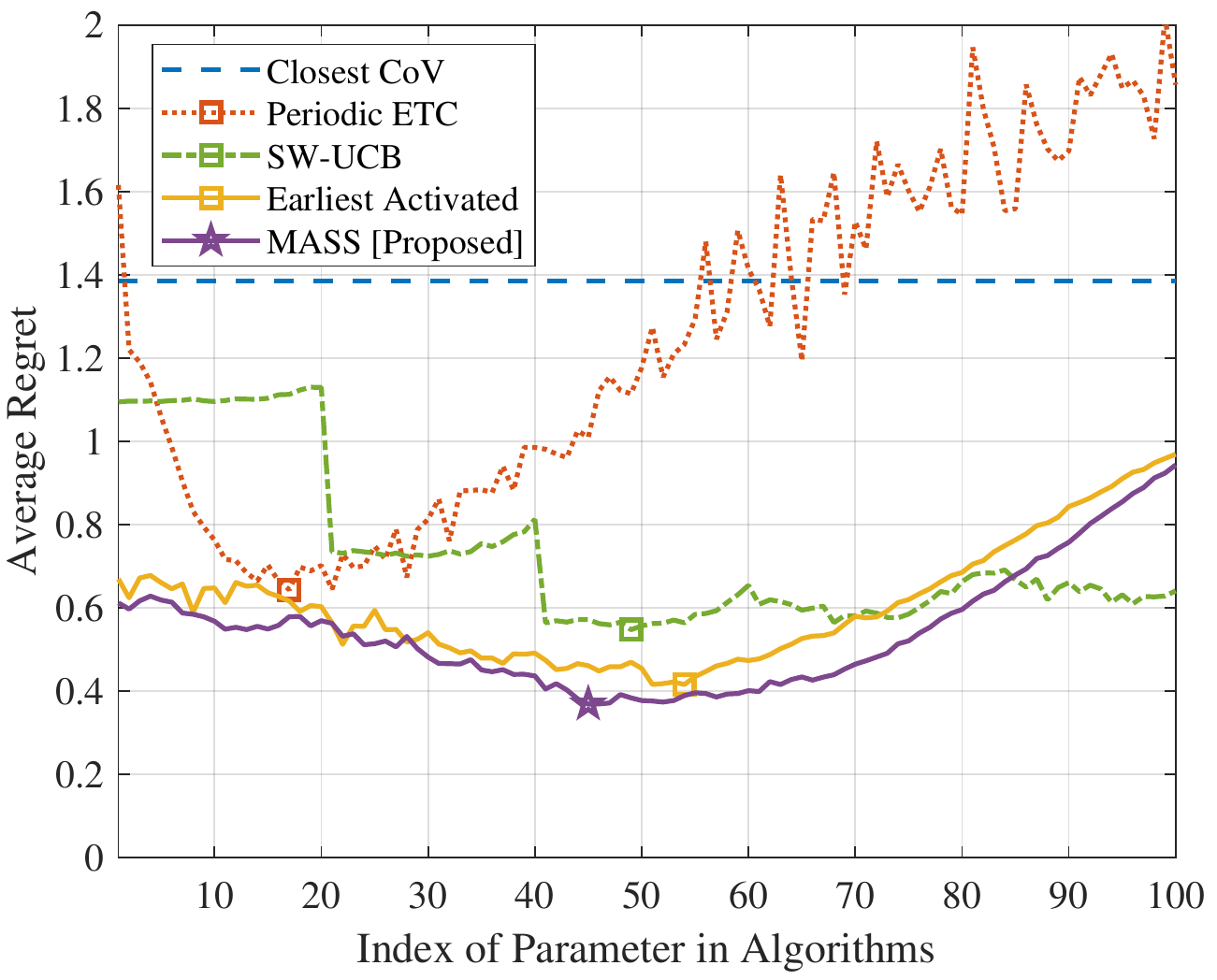}
	\caption{The average regret of different scheduling policies with sweeping algorithm parameters. The ranges of parameters for each algorithm are specified in Table \ref{tab:parameter}. The optimal regrets of scheduling policies are marked on the figure.}
	\label{fig:parameter}
\end{figure}

Fig. \ref{fig:parameter} shows the average learning regret for sweeping algorithm parameters.
The average learning regret reflects the perception cost difference to the offline optimal decision across time.
With reasonable parameters, all the learning-based algorithms, including SW-UCB, Earliest Activated, and the proposed MASS algorithm, outperform the distance-based policy, showing the benefit of learning from historical observations. 
Among the learning-based algorithms, MASS achieves a uniformly better regret performance than other algorithms in a wide range of parameter values.
Furthermore, a well-tuned parameter can minimize the average regret, shown by the markers on the curves.
The optimal parameter of the MASS algorithm is influenced by the rate of change in the perception gain, as stated in Section \ref{sect:performance}.

\begin{figure}[!t]
    \centering
    \subfloat[]{\includegraphics[width=0.4\textwidth]{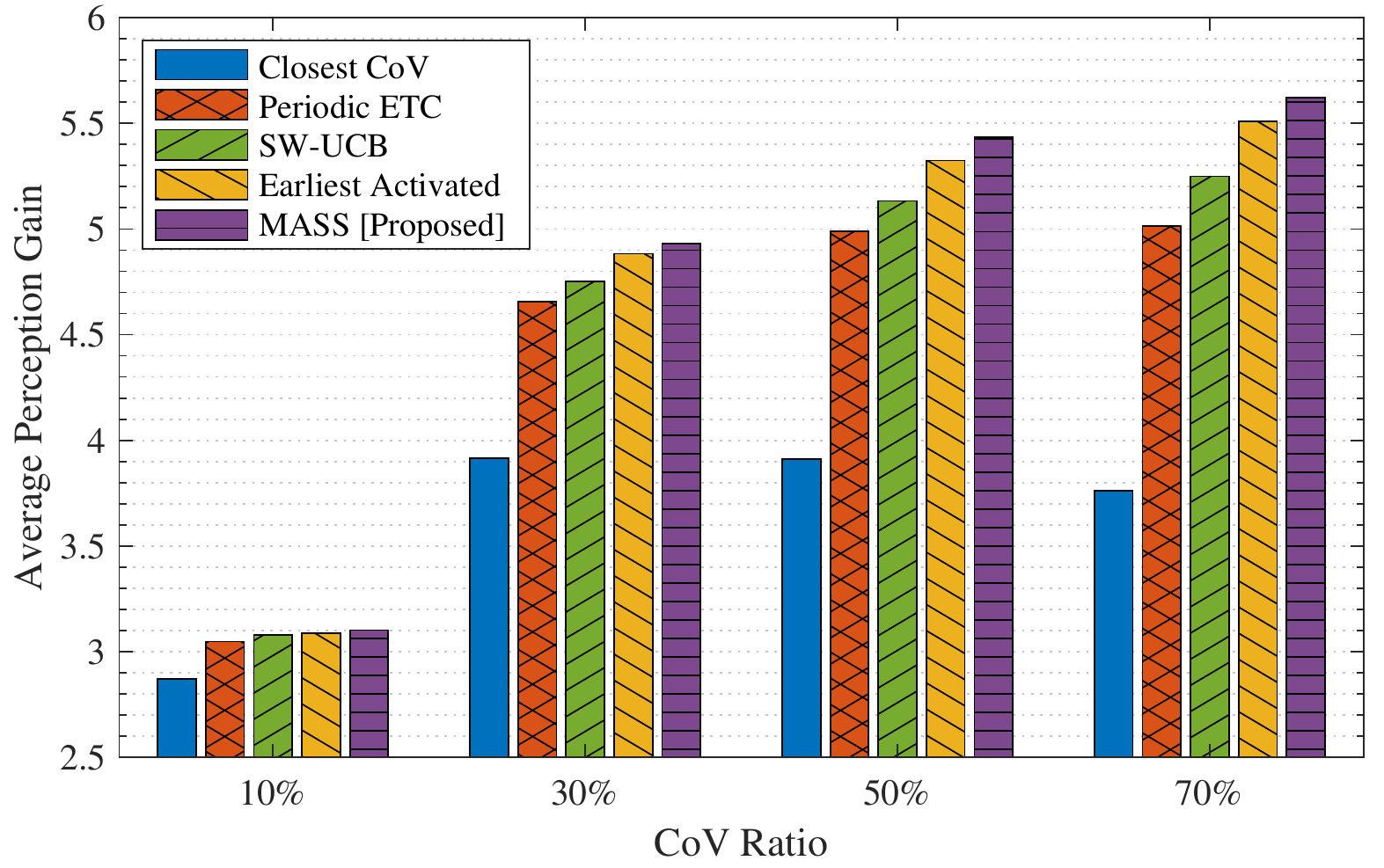}
    \label{fig:mpr-gain}}
    \hfill
    \subfloat[]{\includegraphics[width=0.4\textwidth]{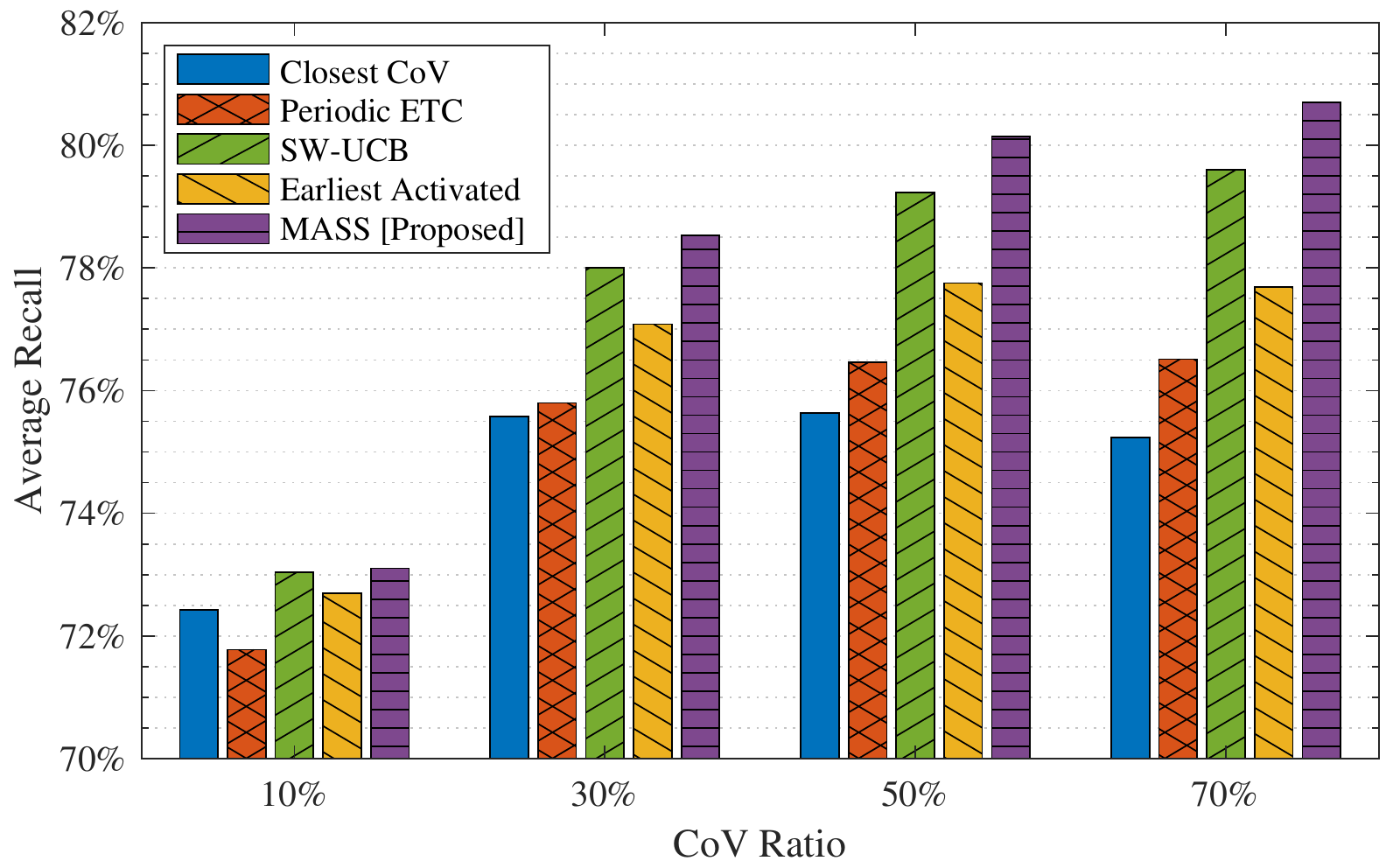}
    \label{fig:mpr-recall}}
    \caption{Comparisons of scheduling algorithms under different CoV ratios. a) The perception gain. b) The recall value.}
    \label{fig:mpr}
\end{figure}

Fig. \ref{fig:mpr} compares the average perception gain of scheduling policies at optimal parameters during the trip under different CoV ratios.
We make three observations as follows.
First, the perception gain generally increases when the CoV ratio is higher.
Typically, when there are more candidate CoVs, the gain from the optimal CoV is higher, and thus the performance is improved when the optimal CoV is exploited. 
With different CoV ratios, the MASS algorithm has stable optimal parameters around $\beta\approx0.6$, showing its robustness to the number of candidates.
Second, the MASS algorithm performs uniformly the best with all CoV ratios. 
The perception gain is improved by up to 49\% compared to position-based policy and 12\% compared to other learning-based algorithms in high CoV ratio settings.
The advantage over SW-UCB implies an essential finding that dynamics is more significant than randomness due to the high-mobility nature of the automated driving scenario.
In contrast to the Periodic ETC algorithm that explores regularly, MASS explores more efficiently by adapting to the actual state.
Finally, although the optimization objective is not precisely aligned to the recall, the MASS algorithm achieves the best perception performance.
The recall is improved by up to 4.2 percentage points compared to other learning-based algorithms in the high CoV ratio setting, which is a considerable gain in the context of automated driving.


\subsection{Case Study: A Trace with LiDAR Frames}

\begin{figure}[!t]
	\centering
	\includegraphics[width=0.4\textwidth]{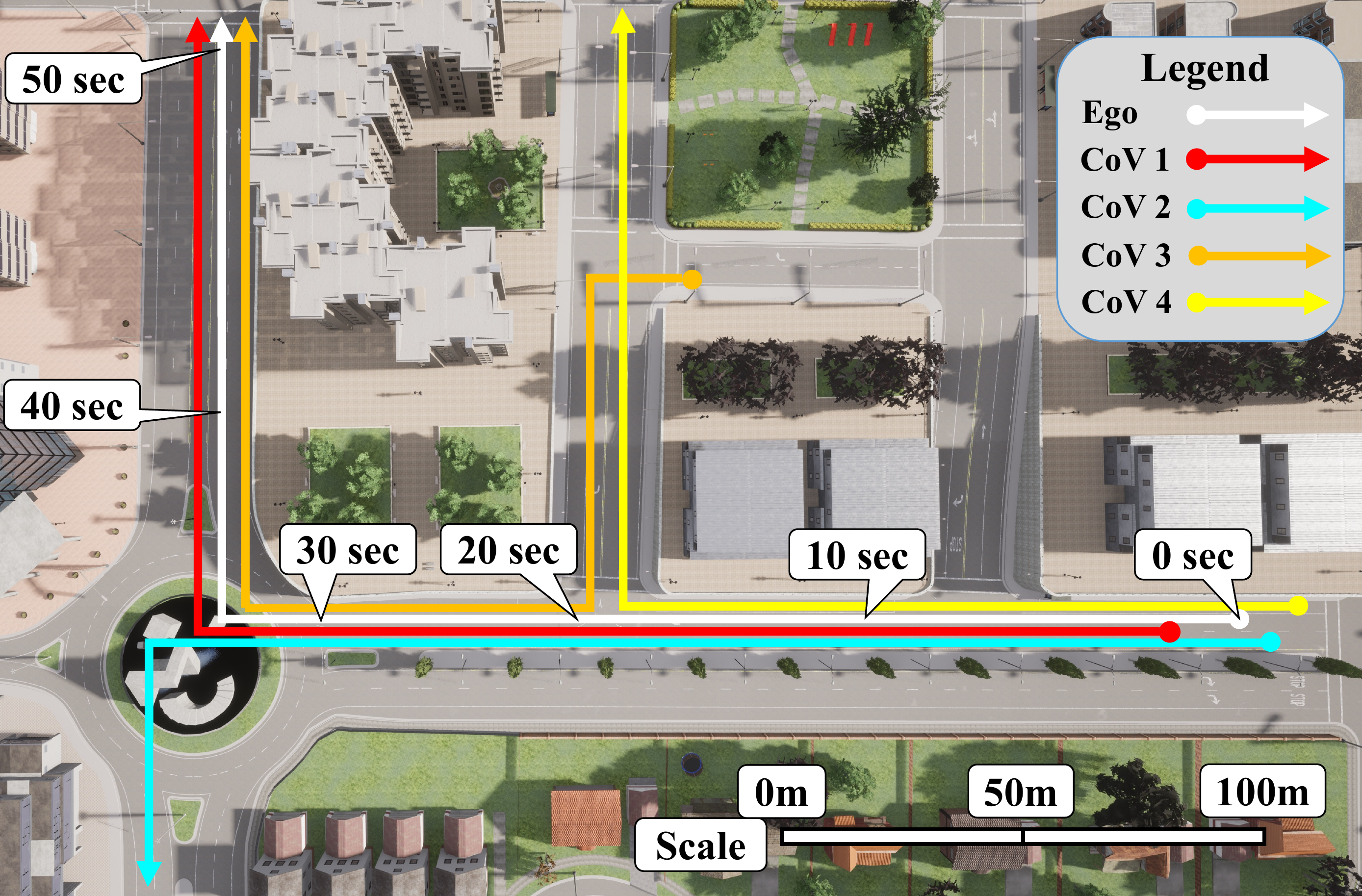}
	\caption{The trajectories of the ego vehicle and other CoVs on the CARLA town map. The position of the ego vehicle is marked every 10 seconds.}
	\label{fig:trace-map}
\end{figure}

To further illustrate the benefits of adaptive exploration, we conduct a case study with a CARLA-generated trace of LiDAR frames.
For a hundred frames, the LiDAR point clouds from the ego vehicle and four other CoVs are simultaneously recorded during a trip of 50 seconds.
The trajectories of the vehicles on the town map are shown in Fig. \ref{fig:trace-map}.
In each time slot, the ego vehicle identifies the CoVs within 100 meters as candidates and selects a CoV to schedule.
For sensor fusion, the received point clouds are transformed into the view of the ego vehicle and then appended together with the onboard sensor data. 
Then it is feed into the 3D object detector, PointPillars, to obtain the perception results.
Several representative algorithms are tested, including the one-shot distance-based Closest CoV, the regularly exploring Periodic ETC, and the adaptively exploring MASS algorithm.
The standalone perception and CP with the offline optimal CoV serve as the upper and lower bounds for the perception cost, respectively.

\begin{figure}[!t]
	\centering
	\includegraphics[width=0.4\textwidth]{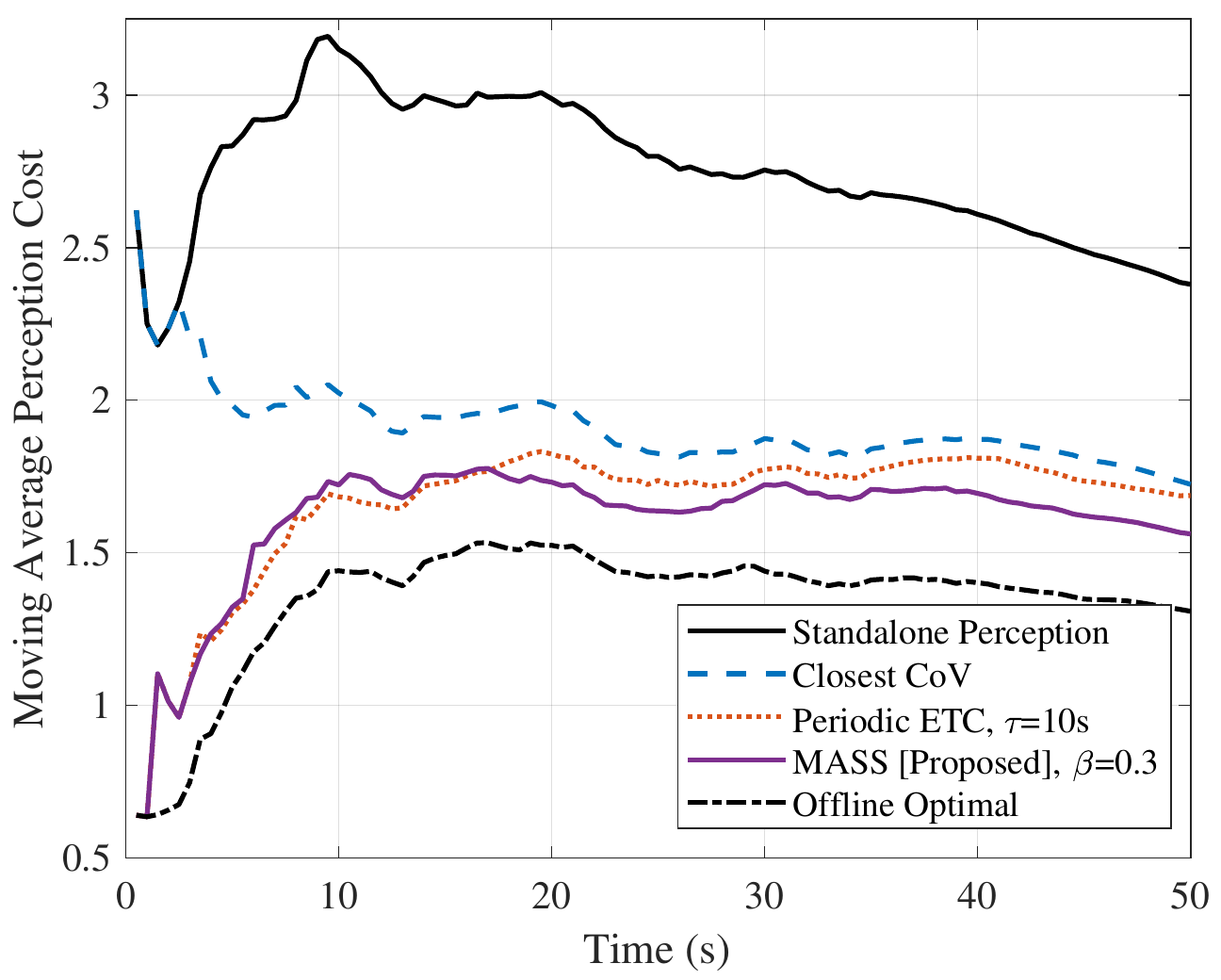}
	\caption{The comparison of moving average perception costs of different scheduling algorithms.}
	\label{fig:trace-cost}
\end{figure}

As shown in Fig. \ref{fig:trace-cost}, the MASS algorithm achieves the best perception quality, consistent with the simulation results.
We further visualize the merged point clouds and detection results, as illustrated in Fig. \ref{fig:lidar}.
At around 38.0s, the ego vehicle is leaving the roundabout, when two cars and one pedestrian are invisible due to the blockage effect.
The CoV 1 in front is scheduled with the Closest CoV policy and the Periodic ETC in the exploitation phase.
However, as shown in Fig. \ref{fig:lidar}(a), there is no traffic in front, and thus the extra sensor has no gain.
Owing to the \emph{adaptive exploration}, the MASS algorithm is aware of the decrease in the leader's gain.
Therefore, the ego vehicle explores and quickly identifies CoV 2 as the new leader since it reveals three additional objects in Fig. \ref{fig:lidar}(b).
Such a process takes full advantage of the information within the historical detection results and learns to make optimal decisions, which is the core intuition of our algorithm.

\begin{figure}[!t]
    \centering
    \subfloat[]{\includegraphics[width=0.4\textwidth]{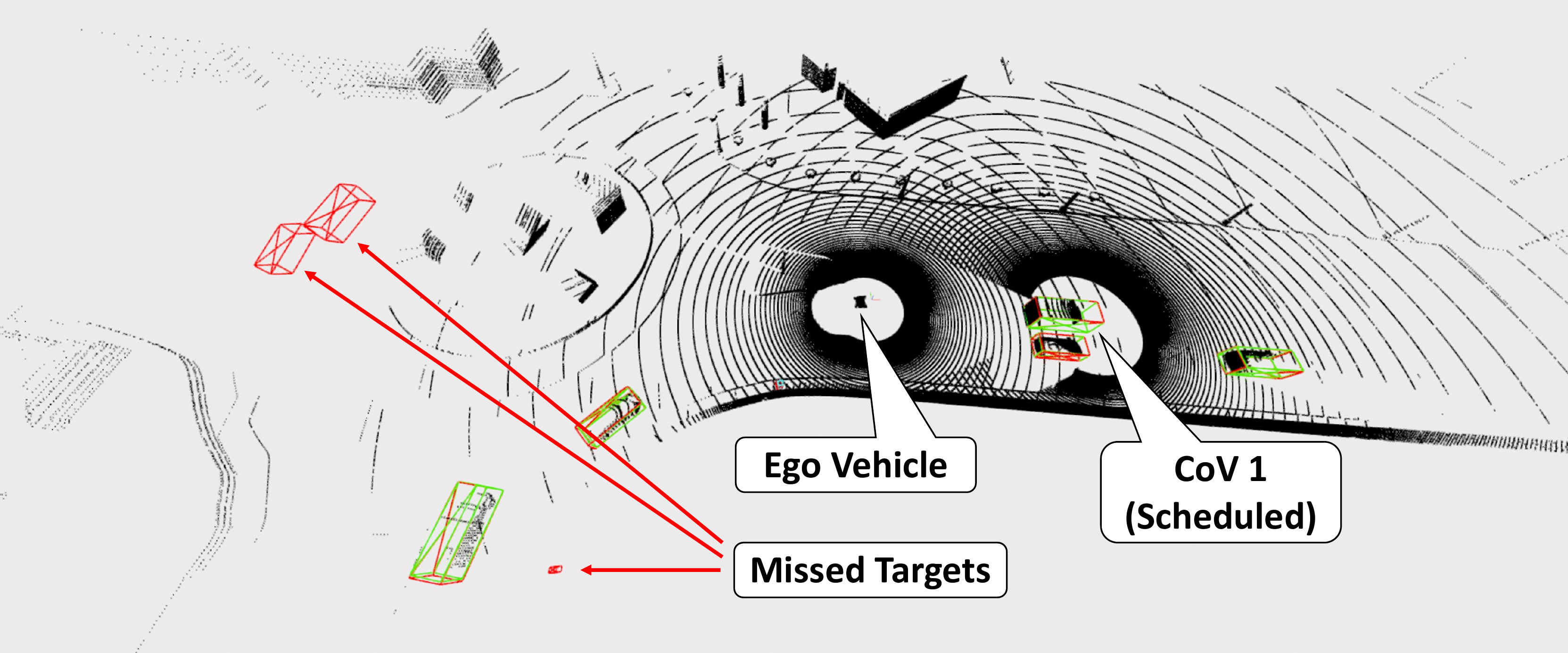}
    \label{fig:lidar-subopt}}
    \hfill
    \subfloat[]{\includegraphics[width=0.4\textwidth]{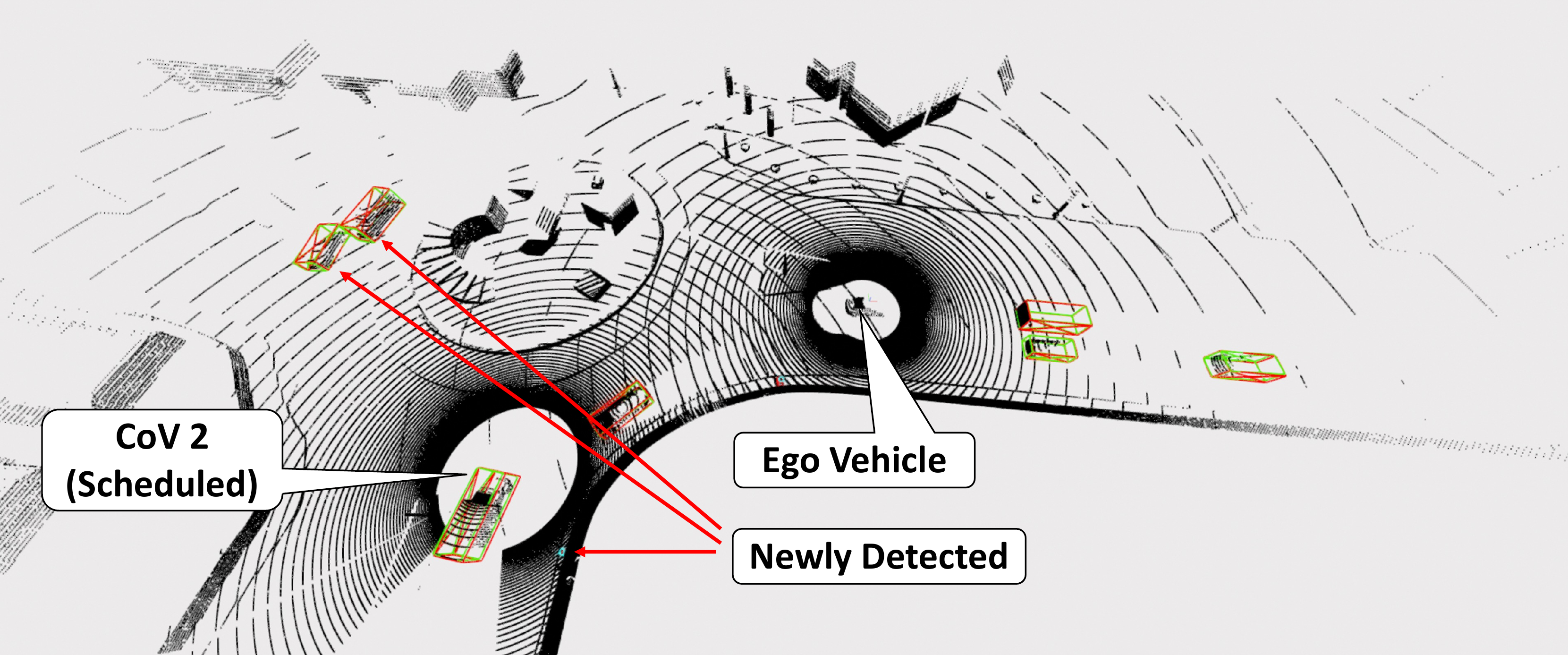}
    \label{fig:lidar-best}}
    \caption{Visualization of the merged point clouds and detection results. (a) The merged point clouds when CoV 1 is scheduled by Closest CoV and Periodic ETC algorithms. (b) The merged point clouds when CoV 2 is scheduled by the proposed MASS algorithm.}
    \label{fig:lidar}
\end{figure}

\section{Conclusion} \label{sect:conclusion}
In this paper, we have studied the scheduling of decentralized cooperative perception within the RMAB framework, fully considering the mobility of CoVs.
A mobility-aware sensor scheduling (MASS) algorithm has been proposed to maximize the average perception gain, leveraging the temporal continuity of perception gains. 
The MASS algorithm enables each CoV to learn the perception gains from candidates in a distributed manner, without the overhead of frequent meta-information exchanges.
An upper bound for the expected average learning regret is provided that matches the lower bound of any online algorithm up to a logarithmic factor.
We have evaluated the MASS algorithm under a realistic SUMO trace, showing that the proposed algorithm can improve the average perception gain by up to 12\% and the recall by up to 4.2 percentage points, compared to other learning-based algorithms.
Finally, a case study has been presented to show the superiority of adaptive exploration from our proposed algorithm.

For future work, we plan to extend the sensor scheduling problem towards more sources, using the combinatorial multi-armed bandit framework \cite{replication}.
Note that the reward is then non-linear since the perception gains of different sources are not independent.
Besides, we will study the uncertainty reduction modeling of the perception gain \cite{uncertainty} to incorporate the influence of false positives in the perception performance.


%

\appendices
\section{Proof of Lemma \ref{lem:wellbehave}} \label{app:wellbehave}
Define the event 
    \begin{align*}
        E_{t,t'}^{(i)} = \left\{|G_i(t)-G_i(t')| \le c_f\sqrt{|t-t'|}\sigma\right\}.
    \end{align*}
    By assumption, $G_i(t)-G_i(t') \sim \mathcal{N}(0,|t-t'|\sigma^2)$, thus
    \begin{align*}
        P(E_{t,t'}^{(i)}) &= 1 - 2 \int_{c_f}^{\infty} \frac{1}{2\pi} e^{-x^2/2} dx \\
        &= 1 - O\left(\log{\sigma^{-1}}\right)^{-1/2} e^{-\frac{9}{2}\log{\sigma^{-1}}} \ge 1 - O(\sigma^{4.4}).
    \end{align*}
    Then
    \begin{align*}
        P(E_t)
        &\ge \prod_{i=1}^2 \prod_{t'=t-\sigma^{-2}}^{t+\sigma^{-2}} P(E_{t,t'}^{(i)})
        > 1 - O(\sigma^{2.4}).
    \end{align*}
    Therefore, $P(\Bar{E}_t) < O(\sigma^{2.4})$.

\section{Proof of Lemma \ref{lem:properties}} \label{app:properties}
In the proposed algorithm, the last-seen time $\tau_i(t)$ must exist and satisfies $t - \tau_i(t) < \sigma^{-2}$.
For simplicity, we denote $\tau_i = \tau_i(t)$ when there is no confusion.

a) Without loss of generality, we assume CoV 1 is optimal at $t$. 
If CoV 2 is scheduled at time $t$, then at $t+1$,
\begin{align*}
    G_1(\tau_1)+\beta\sigma\sqrt{t-\tau_1+1} &\ge G_1(t)-c_f\sigma\sqrt{t-\tau_1} \\ 
    &\quad + \beta\sigma\sqrt{t-\tau_1+1} \\ 
    &> G_2(t)+\beta\sigma.
\end{align*}
Therefore, CoV 1 must be scheduled no later than $t+1$.

b) Without loss of generality, we assume CoV 1 is the leader at $t$.
If CoV 1 is not scheduled at time $t-1$, then
\begin{align*}
    G_1(\tau_1)+\beta\sqrt{t-\tau_1} \ge G_2(t-1)+\beta\sigma,
\end{align*}
since $G_1(\tau_1) \ge G_2(t-1)$, and the proof is completed.
Otherwise, we have CoV 1 scheduled at $t-1$ and CoV 2 scheduled at $t$, then
\begin{align}
    G_1(\tau_1)+\beta\sigma\sqrt{t-\tau_1-1} \ge G_2(\tau_2)+\beta\sigma\sqrt{t-\tau_2-1}. \label{formula:lemma2b}
\end{align}
i) If $\tau_1=t-2$, $\tau_2 \le t-3$, then (\ref{formula:lemma2b}) implies
\begin{align*}
    G_1(t-2)+\beta\sigma \ge G_2(\tau_2)+\beta\sigma\sqrt{t-\tau_2-1}.
\end{align*}
Therefore, at time $t+1$,
\begin{align*}
    G_1(t-1)+\sqrt{2}\beta\sigma &\ge G_1(t-2)+c_f\sigma+\sqrt{2}\beta\sigma \\
    &> G_2(t)+\beta\sigma. 
\end{align*}
ii) If $\tau_2=t-2$, $\tau_1 \le t-3$, then at time $t+1$,
\begin{align*}
    G_1(t-1)+\sqrt{2}\beta\sigma &\ge G_2(t-2) + \sqrt{2}\beta\sigma > G_2(t) + \beta\sigma.
\end{align*}
To sum up, CoV 1 is scheduled no later than $t+1$.

c) The statement is trivially true if the leader stays the same from $t$ to $t+1$.
Otherwise, Let CoV 1 be the leader at $t+1$, CoV 2 be the leader at $t$,
$H^*(t+1)-H^*(t) = G_1(t+1)-G_2(t)$. \\
i) Assume CoV 1 is scheduled at $t$, then CoV 2 must be scheduled at $t-1$.
Therefore,
\begin{align*}
    G_1(t+1)-G_2(t) &\ge G_1(t+1)-G_2(t-1)-c_f\sigma \\
    &\ge G_1(t+1)-G_1(t)-c_f\sigma \ge -2c_f\sigma.
\end{align*}
ii) Assume CoV 2 is scheduled at $t$, we have
\begin{align}
    G_1(\tau_1)+\sqrt{t-\tau_1}\beta\sigma \le G_2(\tau_2)+\sqrt{t-\tau_2}\beta\sigma. \label{formula:lemma2c}
\end{align}
If CoV 1 is scheduled at $t-1$, then
\begin{align*}
    G_1(t+1)-G_2(t) \ge G_1(t+1)-G_1(t-1) \ge -\sqrt{2}c_f\sigma.
\end{align*}
If CoV 2 is scheduled at $t-1$, then
\begin{align*}
    G_1(t+1)-G_2(t) &\ge G_1(\tau_1)-\sqrt{t-\tau_1+1}c_f\sigma-G_2(t) \\
    &> \frac{c_f+\beta}{\beta}(G_1(\tau_1)-G_2(t)) - 2c_f\sigma \\
    &\ge -2c_f\sigma,
\end{align*}
where the second inequality is by (\ref{formula:lemma2c}).

d) Let CoV 1 be the optimal CoV at $t-2$, which is scheduled at $t'\in\{t-2,t-1\}$.
Then either $G_1(t') \le H^*(t')$ or CoV 1 becomes the new leader at $t'+1$.
Either of them leads to
\begin{align*}
    H^*(t'+1) \ge G_1(t')-c_f\sigma.
\end{align*}
If $t'=t-2$, by c) we have 
\begin{align*}
    H^*(t) &\ge H^*(t-1) - 2c_f\sigma \ge G_i(t-2) - 3c_f\sigma  \\
    &= G^*(t-2) - 3c_f\sigma > G^*(t) - 5c_f\sigma.
\end{align*}
If $t'=t-1$,
\begin{align*}
    H^*(t) &= G_i(t-1) - c_f\sigma = G_i(t-2) - 2c_f\sigma \\
    &= G^*(t-2) - 2c_f\sigma > G^*(t) - 4c_f\sigma.
\end{align*}
Therefore, $G^*(t) - H^*(t) < 5c_f\sigma$.

\section{Proof of Lemma \ref{lem:unscheduletime}} \label{app:unscheduletime}
a) We focus on a sub-optimal CoV $i$ and omit the subscript $i$ for simplicity.
Since the problem instance is well-behaved near $t$, for $\forall t'\in[t-\sigma^{-2},t+\sigma^{-2}]$,
\begin{align*}
    |G(t')-G(\tau)| &\le c_f\sqrt{t'-\tau}\sigma, \\
    |H^*(t')-H^*(\tau)| &\le c_f\sqrt{t'-\tau}\sigma + 5c_f\sigma.
\end{align*}
Then
\begin{align*}
    |\delta(t')-\delta(\tau)| \le 2c_f\sqrt{t'-\tau}\sigma + 5c_f\sigma.
\end{align*}
If $\delta(t) \le 20c_f\sigma$, then CoV $i$ is scheduled soon and trivially
\begin{align*}
    t-\tau &= O(1) = \Theta(\delta_i(t)/\beta\sigma)^2, \\
    \delta(\tau) &=O(c_f\sigma) =  \delta(t)+O(c_f\sigma),
\end{align*}
which satisfies the condition.
Otherwise, i) if $\delta(\tau)<\delta(t)/2$, 
\begin{align*}
    2c_f\sqrt{t-\tau}\sigma \ge |\delta(t)-\delta(\tau)|-5c_f\sigma \ge \delta/4,
\end{align*}
and thus $t-\tau\ge\Theta(\delta(t)/c_f\sigma)$.
ii) If $\delta(\tau)\ge\delta(t)/2$, we have for $\forall t' \in [\tau, t+\sigma^{-2}]$,
\begin{align*}
    G(\tau)+\beta\sigma\sqrt{t'-\tau} &= H^*(\tau) - \delta(\tau) + \beta\sigma\sqrt{t'-\tau} \\
    &\le H^*(t') - \delta(t)/2 + (c_f+\beta)\sigma\sqrt{t'-\tau}.
\end{align*}
For any $t'$ satisfying $t'-t \le O(\delta(t)/\beta\sigma)^2$, 
\begin{align*}
    G(\tau)+\beta\sigma\sqrt{t'-\tau} \le H^*(t') + \beta\sigma,
\end{align*}
Therefore, $t-\tau \ge \Theta(\delta(t)/\beta\sigma)$.

During the unscheduled time $t'\in(\tau,t]$, we have
\begin{align*}
    \beta\sigma\sqrt{t'-\tau} &\le H^*(t')+\beta\sigma-G(\tau) \\
    &= \delta(t')+G(t')-G(\tau)+\beta\sigma \\
    &\le \delta(t')+c_f\sigma\sqrt{t'-\tau}+\beta\sigma,
\end{align*}
and consequently $\delta(t')\ge 4c_f\sqrt{t'-\tau}\sigma-\beta\sigma$.
Finally,
\begin{align*}
    |\delta(t')-\delta(\tau)| &\le 2c_f\sqrt{t'-\tau}\sigma + O(c_f\sigma) \\
    &\le \delta(t')/2 + O(c_f\sigma).
\end{align*}
Let $t'=t$, and we obtain $\delta(\tau) \le 2\delta(t) + O(c_f\sigma)$.

b) Assume the CoV 1 is optimal at $t$.
Suppose $\tau_1<t-2$, then CoV 2 is scheduled at $t-1$.
We have
\begin{align*}
    G_1(\tau_1)+\beta\sqrt{t-\tau_1}\sigma &\ge G_1(t)+4\sqrt{3}c_f\sigma \\
    &\ge G_2(t)+c_f\sqrt{2}\sigma+\beta\sigma,
\end{align*}
which contradicts the scheduling decision.
Thus, $\tau_1 \ge t-2$, and
\begin{align*}
    \delta_1(\tau_1) \le 2\sqrt{2}c_f\sigma + 5c_f\sigma = O(c_f\sigma).
\end{align*}

\section{Proof of Theorem \ref{thm:fixed}} \label{app:fixed}
We first bound the expectation of $R^*(T)$.
By (\ref{formula:gain-diff}), define the events 
\begin{align*}
    F_t &= \left\{ G^*(t) - H^*(t) \in (0,5c_f\sigma) \right\}, \\
    \Bar{F}_t &= \left\{ G^*(t) - H^*(t) = 0 \right\}.
\end{align*}
Since $G_i(t)$ has uniform stationary distribution, 
\begin{align*}
    P(F_t) = P\left(|G_1(t)-G_2(t)| \in (0,5c_f\sigma)\right) \le O(c_f\sigma),
\end{align*}
and we obtain
\begin{align*}
    \mathbb{E}[G^*(t)-H^*(t)] &= \mathbb{E}[G^*(t)-H^*(t)|E_t] \cdot P(E_t) + P(\Bar{E}_t) \\
    &\le \mathbb{E}[G^*(t)-H^*(t)|E_t, F_t] \cdot P(E_tF_t) \\
    &\quad + 0 \cdot P(E_t \Bar{F}_t) + P(\Bar{E}_t) \\
    &\le  O(c_f\sigma)^2 + O(\sigma^{2.4}) \le  O(c_f\sigma)^2.
\end{align*}
Next we bound the the expectation of $R_i(T)$. \\
i) For sub-optimal CoV $i$, define the event $\Gamma = \{\delta_i(t)=\delta, G_i(t)<G^*(t)\}$, and we have
\begin{align*}
    \mathbb{E}\left[\frac{\delta_i(t)}{t-\tau_i(t)} | \Gamma \right] &\le \mathbb{E}\left[\frac{\delta_i(t)}{t-\tau_i(t)} | \Gamma E_t \right] + 1\cdot P(\Bar{E}_t|\Gamma) \\
    &\le 
    \begin{cases}
    O\left((\beta\sigma)^2/\delta\right), &\quad \delta > \alpha c_f \sigma, \\
    O(c_f\sigma), &\quad \delta < \alpha c_f \sigma,
    \end{cases}
\end{align*}
where we use (\ref{formula:idle-subopt}), (\ref{formula:delta-subopt}) at the second inequality, and $\alpha$ is a large enough constant.
The distribution of $\delta_i(t)$ is bounded by
\begin{align*}
    &P(\delta_i(t) \le \delta | G_i(t)<G^*(t)=G^*) \\
    \le &P(G^*(t)-G_i(t)\le\delta+O(c_f\sigma)|G_i(t)<G^*)
    \le \frac{\delta+O(c_f\sigma)}{G^*}.
\end{align*}
Integrating over $G^*$,
\begin{align*}
    P(\delta_i(t) \le \delta | G_i(t)<G^*(t)) &= \int_0^1 \frac{\delta+O(c_f\sigma)}{G^*} f(G^*) dG^* \\
    &= (\delta+O(c_f\sigma)) \cdot \mathbb{E}[1/G^*] \\
    &\le \delta+O(c_f\sigma).
\end{align*}
Then,
\begin{align*}
    \mathbb{E}\left[\frac{\delta_i(t)}{t-\tau_i(t)} | G_i(t)<G^*(t) \right]  &\le O(c_f\sigma)^2 + \int_{\alpha c_f\sigma}^1 \frac{(\beta\sigma)^2}{\delta} d\delta \\
    &\le O(\sigma^2 \log^3(1/\sigma)).
\end{align*}
ii) For optimal CoV $i$, define $\Lambda=\left\{G_i(t)=G^*(t)\right\}$.
\begin{align*}
    \mathbb{E}\left[\frac{\delta_i(t)}{t-\tau_i(t)} | \Lambda \right] &\le \mathbb{E}\left[\frac{\delta_i(t)}{t-\tau_i(t)} | F_t,\Lambda \right] \cdot P(F_t) + 0 \cdot P(\Bar{F}_t) \\
    &\le O(\sigma^2 \log^2(1/\sigma)),
\end{align*}
where the second inequality is by (\ref{formula:delta-opt}). To sum up, 
\begin{align*}
    \Bar{R}_\mathcal{\text{MASS}} &= \frac{1}{T} \mathbb{E} \left[ R^*(T)+ \sum_{i=1}^2 R_i(T) \right] \\
    &\le \frac{1}{T} \sum_{t=3}^{T} \mathbb{E} \left[G^*(t)-H^*(t)+ \sum_{i=1}^2 \frac{\delta_i(t)}{t-\tau_i(t)}\right] + \frac{2}{T} \\
    &\le O(\sigma^2 \log^3(1/\sigma)).
\end{align*}

\section{Proof of Theorem \ref{thm:dynamic}} \label{app:dynamic}
Denote the arrival time of new candidates CoV $i$ by $s_1, s_2, \cdots, s_N$. We split the whole trip into periods $[1,s_1-1], [s_1,s_2-1], \cdots, [s_N,T]$.
At the beginning of the $i$-th period $s_i$, if $G_i(s_i)>G_j(s_i)$, we augment the period to $[\tau_j(s_i), s_{i+1}-1]$ so that $\tau_j(t)$ exists within the period for $t>s_i+1$.
Then by Lemma \ref{lem:properties}, CoV $j$ is sub-optimal during $[\tau_j(s_i), s_i-2]$, otherwise CoV $j$ should be scheduled before $s_i$.
Naturally, the scheduling decision is optimal at $[\tau_j(s_i)+1, s_i-2]$.
Therefore, we compute the expected learning regret inside the $i$-th period,
\begin{align}
    R_i \le \sum_{t=s_i+1}^{s_{i+1}-1} \mathbb{E} \left[G^*(t)-H^*(t)+ \sum_{i=1}^2 \frac{\delta_i(t)}{t-\tau_i(t)}\right] + 2. \label{formula:regret-period}
\end{align}
To sum up, we obtain the average expected learning rate using the results in the proof of Theorem \ref{thm:fixed},
\begin{align*}
    \Bar{R}_\mathcal{\text{MASS}} &= \frac{1}{T} \sum_{i=1}^N R_i \\
    &\le \frac{1}{T} \sum_{i=1}^{N} \sum_{t=s_i+1}^{s_{i+1}-1} \mathbb{E} \left[G^*(t)-H^*(t) + \sum_{i=1}^2 \frac{\delta_i(t)}{t-\tau_i(t)}\right] \\
    &\quad + \mathbb{E}\left[\frac{2N}{T}\right] \\
    &\le O\left(\sigma^2\log^3(1/\sigma)\right) + 2\lambda. 
\end{align*}
When $\lambda \le O\left(\sigma^2\log^3(1/\sigma)\right)$, $\Bar{R}_\mathcal{\text{MASS}} \le O\left(\sigma^2\log^3(1/\sigma)\right).$

There is a final note that the number of candidates could be less than two due to the departure of CoVs, when scheduling is trivial.
If $|\mathcal{V}_t| < 2$ at the beginning of a period, then there is no regret during the period.
Otherwise, if any departure happens inside the period, then there is no regret for the rest of the period.
In both cases, the regret is upper bounded by (\ref{formula:regret-period}).




\ifCLASSOPTIONcaptionsoff
  \newpage
\fi

\end{document}